\documentclass{article}
\pdfoutput=1
\usepackage{times}
\usepackage{graphicx} % more modern
\usepackage{subcaption}

\usepackage{algorithm}
\usepackage{algorithmic}

\usepackage{amsmath,amsthm,amssymb,amsfonts,bbm}
\usepackage{bm}
\usepackage{tikz}
\usetikzlibrary{calc, positioning}
\newtheorem{theorem}{Theorem} 

\newtheorem{prop}[theorem]{Proposition}
\newtheorem{hyp}[theorem]{Hypothesis}
\newtheorem{cor}[theorem]{Corollary} 
\usepackage[siunitx]{circuitikz}

\newcommand\numberthis{\addtocounter{equation}{1}\tag{\theequation}}
\usepackage{calc}
\usepackage{float}
\usepackage{hyperref}

\newcommand{\rone}[1]{\textcolor{black}{#1}}
\newcommand{\rtwo}[1]{\textcolor{black}{#1}}

\newcommand{\rsix}[1]{\textcolor{black}{#1}}
\newcommand{\change}[1]{\textcolor{black}{#1}}

\usepackage[accepted]{icml2018}

\icmltitlerunning{Invariance of Weight Distributions in Rectified MLPs}

\begin{document} 
\allowdisplaybreaks
\twocolumn[
\icmltitle{Invariance of Weight Distributions in Rectified MLPs}



% you can specify symbols, otherwise they are numbered in order

\begin{icmlauthorlist}
\icmlauthor{Russell Tsuchida}{to1}
\icmlauthor{Farbod Roosta-Khorasani}{to2,to3}
\icmlauthor{Marcus Gallagher}{to1}
\end{icmlauthorlist}

\icmlaffiliation{to1}{School of ITEE, University of Queensland, Brisbane, Queensland, Australia}
\icmlaffiliation{to2}{School of Mathematics and Physics, University of Queensland, Brisbane, Queensland, Australia}
\icmlaffiliation{to3}{International Computer Science Institute, Berkeley, California, USA}

\icmlcorrespondingauthor{Russell Tsuchida}{s.tsuchida@uq.edu.au}
\icmlcorrespondingauthor{Farbod Roosta-Khorasani}{fred.roosta@uq.edu.au}
\icmlcorrespondingauthor{Marcus Gallagher}{marcusg@uq.edu.au}

\icmlkeywords{Neural Network, kernel, initialization, expressivity}

\vskip 0.3in
]



\printAffiliationsAndNotice{}  % leave blank if no need to mention equal contribution

%\title{Invariance of Weight Distributions in Rectified MLPs}
%\author{
%  Russell Tsuchida \footnote{School of ITEE, University of Queensland, Brisbane, Australia, Email:  s.tsuchida@uq.edu.au} \\
%  \and
%  Farbod Roosta-Khorasani \footnote{School of Mathematics and Physics, University of Queensland, Brisbane, Australia, and International
%Computer Science Institute, Berkeley, USA, Email:  fred.roosta@uq.edu.au} \\
%  \and
%  Marcus Gallagher \footnote{School of ITEE, University of Queensland, Brisbane, Australia, Email:  marcusg@uq.edu.au} \\
%}
%maketitle

\begin{abstract} 
An interesting approach to analyzing neural networks that has received renewed attention is to examine the equivalent kernel of the neural network. This is based on the fact that a fully connected feedforward network with one hidden layer, a certain weight distribution, an activation function, and an infinite number of neurons can be viewed as a mapping into a Hilbert space. We derive the equivalent kernels of MLPs with ReLU or Leaky ReLU activations for all rotationally-invariant weight distributions, generalizing a previous result that required Gaussian weight distributions. Additionally, the Central Limit Theorem is used to show that for certain activation functions, kernels corresponding to layers with weight distributions having $0$ mean and finite absolute third moment are asymptotically universal, and are well approximated by the kernel corresponding to layers with spherical Gaussian weights. In deep networks, as depth increases the equivalent kernel approaches a pathological fixed point, which can be used to argue why training randomly initialized networks can be difficult. Our results also have implications for weight initialization.
\end{abstract} 

\section{Introduction}
\label{s:intro}
Neural networks have recently been applied to a number of diverse problems with impressive results~\cite{vanwavenet, silver2017mastering,berthelot2017began}. These breakthroughs largely appear to be driven by application rather than an understanding of the capabilities and training of neural networks. Recently, significant work has been done to increase understanding of neural networks~\cite{choromanska2015loss, haeffele2015global, poole2016exponential, schoenholz2016deep, zhang2016understanding, martin2017rethinking, shwartz2017opening, pmlr-v70-balduzzi17b, pmlr-v70-raghu17a}. However, there is still work to be done to bring theoretical understanding in line with the results seen in practice.

The connection between neural networks and kernel machines has long been studied~\cite{neal1995bayesian}. Much past work has been done to investigate the equivalent kernel of certain neural networks, either experimentally~\cite{burgess1997estimating}, through sampling ~\cite{sinha2016learning, livni2017learning, lee2017gp}, or analytically by assuming some random distribution over the weight parameters in the network~\cite{williams1997computing, cho2009kernel, pandey2014go, pandey2014learning, daniely2016toward, JMLR:v18:14-546}. Surprisingly, in the latter approach, rarely have distributions other than the Gaussian distribution been analyzed. This is perhaps due to early influential work on Bayesian Networks~\cite{mackay1992practical}, which laid a strong mathematical foundation for a Bayesian approach to training networks. Another reason may be that some researchers may hold the intuitive (but \emph{not necessarily principled}) view that the Central Limit Theorem (CLT) should somehow apply.

In this work, we investigate the equivalent kernels for networks with Rectified Linear Unit (ReLU), Leaky ReLU (LReLU) or other activation functions, one-hidden layer, and more general weight distributions. Our analysis carries over to deep networks. We investigate the consequences that weight initialization has on the equivalent kernel at the beginning of training. While initialization schemes that mitigate exploding/vanishing gradient problems~\cite{hochreiter1991untersuchungen, bengio1994learning, hochreiter2001gradient} for other activation functions and weight distribution combinations have been explored in earlier works~\cite{glorot2010understanding, he2015delving}, we discuss an initialization scheme for Muli-Layer Perceptrons (MLPs) with LReLUs and weights coming from distributions with $0$ mean and finite absolute third moment. The derived kernels also allow us to analyze the loss of information as an input is propagated through the network, offering a complementary view to the shattered gradient problem~\cite{pmlr-v70-balduzzi17b}. 

\section{Preliminaries}
Consider a fully connected (FC) feedforward neural network with $m$ inputs and a hidden layer with $n$ neurons. Let $\sigma: \mathbb{R} \to \mathbb{R}$ be the activation function of all the neurons in the hidden layer. \iffalse We have a single hidden layer network, as shown in Figure~\ref{fig:network}.
\begin{figure}[!htbp]
\centering
\begin{tikzpicture}[
roundnode/.style={circle, draw=black!, fill=white!5, very thick, minimum size=9mm},
]
% First layer
%Nodes
\node				    (0)        				{};
\node[roundnode]        (1)       [right=of 0]	{$x_1$};
\node[roundnode]      	(2)       [right=of 1] 	{$x_2$};
\node[roundnode]    	(3)       [right=of 2]	{$x_m$};

% Dots
\node (4) at ($(0)!.5!(1)$) {};
\node (5) at ($(1)!.5!(2)$) {};
\node (6) at ($(2)!.5!(3)$) {\ldots};

%Second layer
\node[roundnode]        (7)       [above=of 4]    		{$\sigma_1$};
\node[roundnode]        (8)       [above=of 5]    		{$\sigma_2$};
\node[roundnode]        (9)       [right=of 8]    		{$\sigma_3$};
\node[roundnode]        (10)      [right=of 9]  	  	{$\sigma_n$};

%Dots
\node (11) at ($(9)!.5!(10)$) {\ldots};

%Lines
\foreach \x in {1, ..., 3}
	\foreach \y in {7, ..., 10}
		\draw[line width=1.0] (\x.north) -- (\y.south);

%\draw[->] (rightsquare.south) .. controls +(down:7mm) and +(right:7mm) %.. (lowercircle.east);
\end{tikzpicture}
\caption{An FC feedforward network with one hidden layer.}
\label{fig:network}
\end{figure}
\fi
Further assume that the biases are $0$, as is common when initializing neural network parameters. For any two inputs $\mathbf{x}, \mathbf{y} \in \mathbb{R}^m$ propagated through the network, the dot product in the hidden layer is
\begin{align}
\label{eq:sum}
\frac{1}{n}\bm{h}(\mathbf{x})\cdot \bm{h}(\mathbf{y}) &= \frac{1}{n} \sum_{i=1}^n \sigma( \mathbf{w_i} \cdot \mathbf{x} ) \sigma( \mathbf{w_i} \cdot \mathbf{y} ),
\end{align}
where $\bm{h}( \cdot )$ denotes the $n$ dimensional vector in the hidden layer and $\mathbf{w_i} \in \mathbb{R}^m$ is the weight vector into the $i^{th}$ neuron.
Assuming an infinite number of hidden neurons, the sum in~\eqref{eq:sum} has an interpretation as an inner product in feature space, which corresponds to the kernel of a Hilbert space. We have
\begin{align}
\label{kernel}
k(\mathbf{x}, \mathbf{y}) &= \int_{\mathbb{R}^m} \sigma( \mathbf{w} \cdot \mathbf{x} ) \sigma( \mathbf{w} \cdot \mathbf{y} ) f(\mathbf{w}) \,\mathbf{dw},
\end{align}
where $f(\mathbf{w})$ is the probability density function (PDF) for the identically distributed weight vector $\bm{W}=(W_1, ..., W_m)^T$ in the network. The connection of~\eqref{kernel} to the kernels in kernel machines is well-known~\cite{neal1995bayesian,williams1997computing,cho2009kernel}. 

\rsix{Probabilistic bounds for the error between~\eqref{eq:sum} and~\eqref{kernel} have been derived in special cases~\cite{rahimi2008random} when the kernel is shift-invariant. Two specific random feature mappings are considered: \textbf{(1)} Random Fourier features are taken for the $\sigma$ in \eqref{eq:sum}. Calculating the approximation error in this way requires being able to sample from the PDF defined by the Fourier transform of the target kernel. More explicitly, the weight distribution $f$ is the Fourier transform of the target kernel and the $n$ samples $\sigma(\mathbf{w_i} \cdot \mathbf{x})$ are replaced by some appropriate scale of $\cos(\mathbf{w_i} \cdot \mathbf{x})$. \textbf{(2)} A random bit string $\sigma(\mathbf{x_i})$ is associated to each input according to a grid with random pitch $\delta$ sampled from $f$ imposed on the input space. This method requires having access to the second derivative of the target kernel to sample from the distribution $f$.}

\rsix{Other work~\cite{bach2017equivalence} has focused on the smallest error between a target function $g$ in the reproducing kernel Hilbert space (RKHS) defined by~\eqref{kernel} and an approximate function $\hat{g}$ expressible by the RKHS with the kernel~\eqref{eq:sum}. More explicitly, let $g(x) = \int_{\mathbb{R}^m} G(\mathbf{w}) \sigma(\mathbf{w}, \mathbf{x}) f(\mathbf{w}) \,d\mathbf{w}$ be the representation of $g$ in the RKHS. The quantity $\big\Vert \hat{g} - g \big\Vert = \big\Vert \sum_{i=1}^n \alpha_i \sigma(\mathbf{w_i}, \cdot) - \int_{\mathbb{R}^m} G(\mathbf{w}) \sigma(\mathbf{w}, \cdot)  f(\mathbf{w}) \,d\mathbf{w} \big\Vert$ (with some suitable norm) is studied for the best set of $\alpha_i$ and random $\mathbf{w_i}$ with an optimized distribution.}

\rsix{Yet another measure of kernel approximation error is investigated by Rudi \& Rosasco~\yrcite{rudi2017generalization}. Let $\hat{g}$ and $g$ be the optimal solutions to the ridge regression problem of minimizing a regularized cost function $C$ using the kernel~\eqref{eq:sum} and the kernel~\eqref{kernel} respectively. The number of datapoints $n$ required to probabilistically bound $C(\hat{g}) - C(g)$ is found to be $O(\sqrt{n} \log n)$ under a suitable set of assumptions. This work notes the connection between kernel machines and one-layer Neural Networks with ReLU activations and Gaussian weights by citing Cho \& Saul~\yrcite{cho2009kernel}. We extend this connection by considering other weight distributions and activation functions.}

\rsix{In this work our focus is on deriving expressions for the target kernel, not the approximation error. Additionally, we consider random mappings that have not been considered elsewhere. Our work is related to work by Poole et al.~\yrcite{poole2016exponential} and Schoenholz et al.~\yrcite{schoenholz2016deep}. However, our results apply to the unbounded (L)ReLU activation function and more general weight distributions, and their work considers random biases as well as weights.}

\section{Equivalent Kernels for Infinite Width Hidden Layers}
The kernel~\eqref{kernel} has previously been evaluated for a number of choices of $f$ and $\sigma$ \cite{williams1997computing, le2007continuous, cho2009kernel, pandey2014go, pandey2014learning}. In particular, the equivalent kernel for a one-hidden layer network with spherical Gaussian weights of variance $\mathbb{E}[W_i^2]$ and mean $0$ is the Arc-Cosine Kernel~\cite{cho2009kernel}
\iffalse
The known kernels are summarized in Table~\ref{tab:kernels}.

\iffalse
\begin{tabular}{ | c | c | c | } \hline
  & \textbf{Gaussian} & \textbf{Uniform} \hline \\ 
 \textbf{Gaussian} & \cite{williams1997computing} & \\
 \textbf{Sigmoid} & \cite{williams1997computing} &  \\
 \textbf{sgn} & & \cite{le2007continuous} \\  
 \textbf{ReLU} &  \cite{cho2009kernel, pandey2014learning, pandey2014go} &   \hline 
\end{tabular}
\fi

\begin{table}[!htbp]
\centering
\caption{Known equivalent kernels.}
\begin{tabular}{ | c | c | c |} \hline
  & \textbf{Gaussian} & \textbf{Uniform} \\ \hline
  \textbf{Gaussian} & \checkmark & \\
  \textbf{Sigmoid} & \checkmark & \\
  \textbf{sgn} & & \checkmark \\
  \textbf{ReLU} & \checkmark & \\ \hline
\end{tabular}
\label{tab:kernels}
\end{table}
\fi
\begin{equation}
k(\mathbf{x}, \mathbf{y}) = \frac{\mathbb{E}[W_i^2] \Vert \mathbf{x} \Vert \Vert \mathbf{y} \Vert }{2\pi} \big( \sin\theta_0 + (\pi-\theta_0)\cos\theta_0 \big),
\label{eq:chokernel}
\end{equation}
where $\theta_0 = \cos^{-1} \big( \frac{\mathbf{x} \cdot \mathbf{y}}{\Vert \mathbf{x} \Vert \Vert \mathbf{y} \Vert} \big)$ is the angle between the inputs $\mathbf{x}$ and $\mathbf{y}$ and $\Vert \cdot \Vert$ denotes the $\ell^2$ norm. Noticing that the Arc-Cosine Kernel $k(\mathbf{x}, \mathbf{y})$ depends on $\mathbf{x}$ and $\mathbf{y}$ only through their norms, with an abuse of notation we will henceforth set $k(\mathbf{x}, \mathbf{y}) \equiv k(\theta_0).$ Define the \textit{normalized kernel} to be the cosine similarity between the signals in the hidden layer. The normalized Arc-Cosine Kernel is given by
\small
$$\cos\theta_1 =\frac{k(\mathbf{x}, \mathbf{y})}{\sqrt{k(\mathbf{x}, \mathbf{x})} \sqrt{k(\mathbf{y}, \mathbf{y})}} = \frac{1}{\pi} \big( \sin\theta_0 + (\pi - \theta_0) \cos\theta_0 \big),$$
\normalsize
where $\theta_1$ is the angle between the signals in the first layer. Figure~\ref{fig:arccosker} shows a plot of the normalized Arc-Cosine Kernel. 
\vspace{-1.8em}
\begin{figure}[!htbp]
\centering
\includegraphics[scale=0.12]{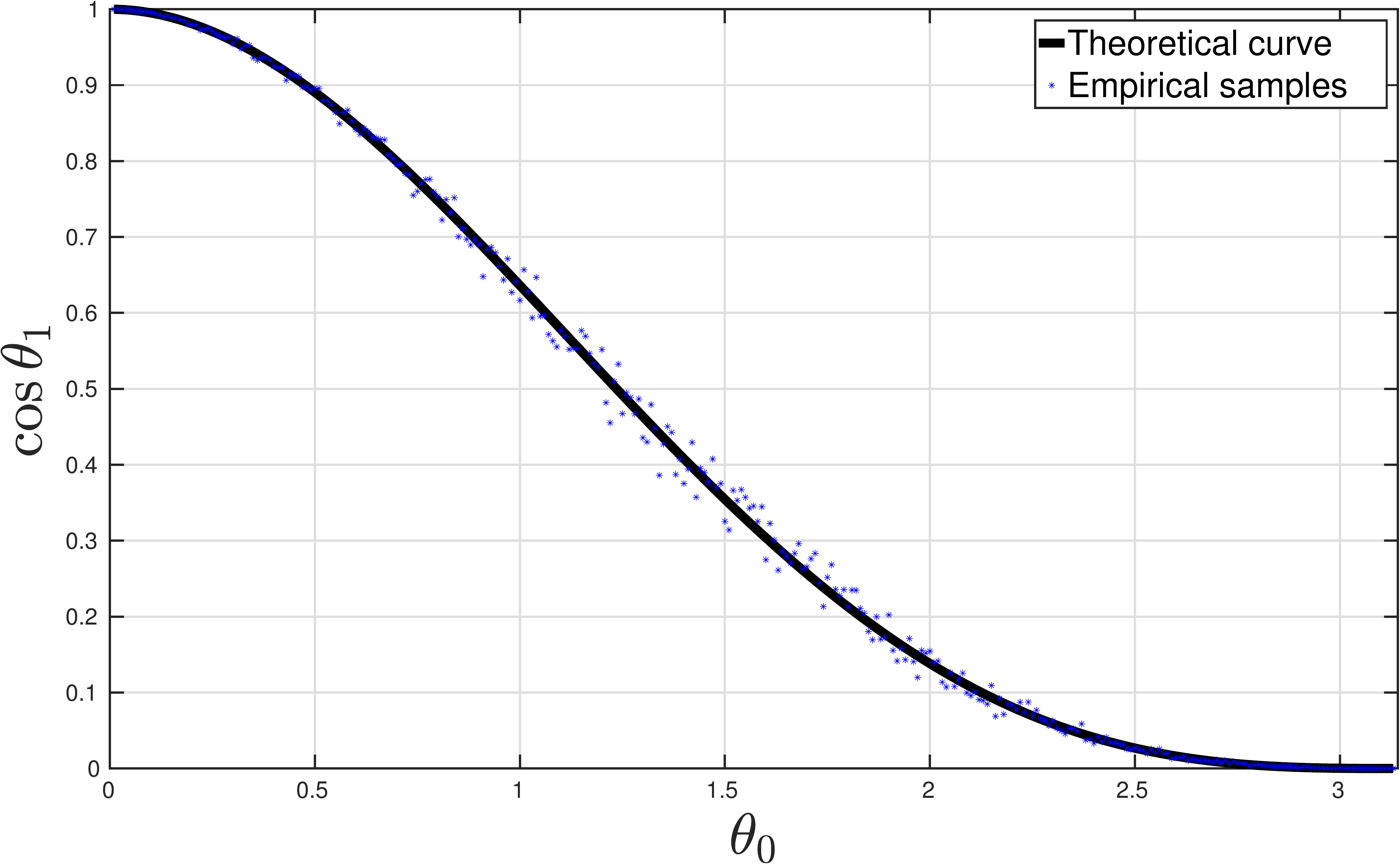}
\caption{Normalized Arc-Cosine Kernel as a function of $\theta_0$ for a single hidden layer network, Gaussian weights, and ReLU activations. Empirical samples from a network with $1000$ inputs and $1000$ hidden units are plotted alongside the theoretical curve. Samples are obtained by generating $R$ from a $QR$ decomposition of a random matrix, then setting $\mathbf{x}=R^T(1, 0,...,0)^T$ and $\mathbf{y}=R^T(\cos\theta, \sin\theta, 0,...,0)^T$.}
\label{fig:arccosker}
\end{figure}

One might ask how the equivalent kernel changes for a different choice of weight distribution. We investigate the equivalent kernel for networks with (L)ReLU activations and general weight distributions in Section~\ref{sec:rotinvar} and~\ref{sec:universal}. The equivalent kernel can be composed and applied to deep networks. The kernel can also be used to choose good weights for initialization. These, as well as other implications for practical neural networks, are investigated in Section~\ref{sec:cor}.

\subsection{Kernels for Rotationally-Invariant Weights}
In this section we show that~\eqref{eq:chokernel} holds more generally than for the case where $f$ is Gaussian. Specifically,~\eqref{eq:chokernel} holds when $f$ is any rotationally invariant distribution. We do this by casting~\eqref{kernel} as the solution to an ODE, and then solving the ODE. We then extend this result using the same technique to the case where $\sigma$ is LReLU.

A rotationally-invariant PDF one with the property $f(\mathbf{w}) = f(R\mathbf{w}) = f(\Vert \mathbf{w} \Vert)$ for all $\mathbf{w}$ and orthogonal matrices $R$. Recall that the class of rotationally-invariant distributions~\cite{bryc1995rotation}, as a subclass of elliptically contoured distributions~\cite{johnson2013multivariate}, includes the Gaussian distribution, the multivariate t-distribution, the symmetric multivariate Laplace distribution, and symmetric multivariate stable distributions.
\label{sec:rotinvar}
\begin{prop}
\label{cor:kernel}
Suppose we have a one-hidden layer feedforward network with ReLU $\sigma$ and random weights $\mathbf{W}$ with uncorrelated and identically distributed rows with rotationally-invariant PDF $f:\mathbb{R}^m \to \mathbb{R}$ and $\mathbb{E}[W_i^2] < \infty$. The equivalent kernel of the network is~\eqref{eq:chokernel}.
\end{prop}
\begin{proof}
First, we require the following.
\begin{prop}
\label{lemma}
With the conditions in Proposition~\ref{cor:kernel} and inputs $\mathbf{x}, \mathbf{y} \in \mathbb{R}^m$ the equivalent kernel of the network is the solution to the Initial Value Problem (IVP)
\begin{equation}
\label{eq:ODE}
k''(\theta_0) + k(\theta_0)=F(\theta_0), \quad k'(\pi) = 0, \quad k(\pi) =0,
\end{equation}
where $\theta_0 \in (0, \pi)$ is the angle between the inputs $\mathbf{x}$ and $\mathbf{y}$. The derivatives are meant in the distributional sense; they are functionals applying to all test functions in $C_c^\infty(0, \pi)$. $F(\theta_0)$ is given by the $m-1$ dimensional integral
\begin{align*}
F(\theta_0)&= \int_{\mathbb{R}^{m-1}}f \Big( (s \sin\theta_0, -s\cos\theta_0, w_3, ..., w_m)^T \Big) \\
&\quad\Theta(s) s^3 \,ds\,dw_3\,dw_4...\,dw_m \Vert \mathbf{x} \Vert \Vert \mathbf{y} \Vert \sin\theta_0, \numberthis \label{eq:forcing_term}
\end{align*}
where $\Theta$ is the Heaviside step function.
\end{prop}
The proof is given in Appendix~\ref{app:lemma}. The main idea is to rotate $\mathbf{w}$ (following Cho \& Saul~\yrcite{cho2009kernel}) so that 
\begin{align*}
k(\mathbf{x}, \mathbf{y})&= \int_{\mathbb{R}^m}\Theta(w_1 ) \Theta(w_1 \cos\theta_0 + w_2 \sin\theta_0) w_1 \\
&\quad(w_1 \cos\theta_0 + w_2 \sin\theta_0)f(\mathbf{w}) \,d\mathbf{w} \Vert \mathbf{x} \Vert \Vert \mathbf{y} \Vert.
\end{align*}
Now differentiating twice with respect to $\theta_0$ yields the second order ODE~\eqref{eq:ODE}. 
\iffalse
The ODE in~\eqref{eq:ODE} has a physical interpretation as a second order circuit with $\theta_0$ replacing time, as shown in Figure~\ref{fig:circuit}. 
\vspace{-1em}
\begin{figure}[!htbp]
\centering
\newcommand{\size}{1.5}
\begin{circuitikz} \draw
(0,0) to[sinusoidal voltage source, label=$F(\theta_0)$]  (0,\size)
      to[cute inductor, label=1<\henry>] (\size, \size) 
      to node[right,pos=1.25]{\small $q_+(\theta_0)$} (\size, \size*3/4)
      to[C,label=1<\farad>] (\size,\size*1/4)
      to node[right,pos=-0.25]{\small $q_-(\theta_0)$} (\size, 0)
      -- (0,0);
\end{circuitikz}
\caption{Kernel ODE as an LC circuit containing a unit inductor, unit capacitor and voltage source $F(\theta_0)$. The kernel $k(\theta_0)$ is the charge across the capacitor $k(\theta_0)=q_+(\theta_0)-q_-(\theta_0)$.}
\label{fig:circuit}
\end{figure}
\fi
The usefulness of the ODE in its current form is limited, since the forcing term $F(\theta_0)$ as in~\eqref{eq:forcing_term} is difficult to interpret.
However, regardless of the underlying distribution on weights $ \mathbf{w} $, as long as the PDF $ f $ in~\eqref{eq:forcing_term} corresponds to any rotationally-invariant distribution, the integral enjoys a much simpler representation. 
\begin{prop}
\label{theorem}
With the conditions in Proposition~\ref{cor:kernel}, the forcing term $F(\theta_0)$ in the kernel ODE is given by $F(\theta_0) = K \sin\theta_0,$ where
\begin{align*}
K &= \int_{\mathbb{R}^{m-1}} \Theta(s) s^3 f \big( (s, 0, w_3, ..., w_m)^T \big) \\
&\quad \,ds\,dw_3,...\,dw_m \Vert \mathbf{x} \Vert \Vert \mathbf{y} \Vert < \infty,
\end{align*}
and the solution to the distributional ODE~\eqref{eq:ODE} is the solution to the corresponding classical ODE.
\end{prop}
The proof is given in Appendix~\ref{app:theorem}. 

Note that in the representation $F(\theta_0) = K \sin\theta_0$ of the forcing term, the underlying distribution appears only as a constant $K$. For all rotationally-invariant distributions, the forcing term in~\eqref{eq:ODE} results in an equivalent kernel with the \emph{same form}. We can combine Propositions~\ref{lemma} and~\ref{theorem} to find the equivalent kernel assuming rotationally-invariant weight distributions.

Due to the rotational invariance of $f$, $k(\pi)= \int_{\mathbb{R}^m} \Theta(w_1)w_1^2 f(R\mathbf{w}) \,d\mathbf{w} \Vert \mathbf{x} \Vert \Vert \mathbf{y} \Vert = \frac{\Vert \mathbf{x} \Vert \Vert \mathbf{y} \Vert \mathbb{E}[W_i^2]}{2}$. The solution to the ODE in Proposition~\ref{lemma} using the forcing term from Proposition~\ref{theorem} is $k(\theta_0)=c_1\cos\theta_0 + c_2\sin\theta_0-\frac{1}{2}K\theta_0\cos\theta_0.$ Using the conditions from the IVP and $k(\pi)$, the values of $c_1, c_2$ and $K$ give the required result.
\end{proof}

\iffalse
The Moore-Aronszajn theorem says that each reproducing kernel Hilbert space has associated with it exactly one kernel, and each kernel has associated with it exactly one reproducing kernel Hilbert space~\cite{aronszajn1950theory}. Our result does not contradict this; every kernel may have more than one corresponding feature map.
\fi

One can apply the same technique to the case of LReLU activations $\sigma(z) = \big( a+(1-a)\Theta(z) \big)z,$ where $a$ specifies the gradient of the activation for $z < 0$.  \iffalse A single layer LReLU network does not strictly have an equivalent kernel, since it is naturally implicit that a kernel must be positive semi-definite (PSD). Nevertheless, with LReLU activations and weights coming from a rotationally-invariant distribution, the integral in~\eqref{kernel} has an equivalent form.\fi 
\begin{prop}
\label{cor:lrelukernel}
\change{Consider the same situation as in Proposition~\ref{cor:kernel} with the exception that the activations are LReLU.} The integral~\eqref{kernel} is then given by
\begin{align*}
k(\mathbf{x}, \mathbf{y})&= \Big[ \frac{(1-a)^2}{2\pi} \big(\sin\theta_0 + (\pi-\theta_0)\cos\theta_0 \big)+ a\cos\theta_0 \Big] \\
&\quad\mathbf{E}[W_i^2] \Vert \mathbf{x} \Vert \Vert \mathbf{y} \Vert, \label{eq:kernel} \numberthis
\end{align*}
where $a \in [0,1)$ is the LReLU gradient parameter.
\end{prop}
\iffalse
The corresponding normalized kernel is
\begin{align*}
& \cos(\theta_1) = \\ &\frac{1}{1+a^2}\Big[ \frac{(1-a)^2}{\pi} \big(\sin\theta_0 + (\pi-\theta_0)\cos\theta_0 \big)+ 2a\cos\theta_0 \Big].
\end{align*}
While a one-hidden layer LReLU network's kernel cannot be PSD (with $a>0$), a multi-hidden layer network's kernel may be PSD, as verified later in Figure~\ref{fig:composed}. 
\fi
This is just a slightly more involved calculation than the ReLU case; we defer our proof to the supplementary material.
\subsection{Asymptotic Kernels}
\label{sec:universal}
\rtwo{In this section we approximate $k$ for large $m$ and more general weight PDFs. We invoke the CLT as $m \to \infty$, which requires a condition that we discuss briefly before presenting it formally. The dot product $\mathbf{w}{\cdot}\mathbf{x}$ can be seen as a linear combination of the weights, with the coefficients corresponding to the coordinates of $\mathbf{x}$. Roughly, such a linear combination will obey the CLT if many coefficients are non-zero. To let $m \to \infty$, we \emph{construct} a sequence of inputs $\{ \mathbf{x}^{(m)}\}_{m=2}^\infty$. This may appear unusual in the context of neural networks, since $m$ is fixed and finite in practice. The sequence is used only for asymptotic analysis.}

\rtwo{As an example if the dataset were CelebA~\cite{liu2015faceattributes} with $116412$ inputs, one would have $\mathbf{x}^{(116412)}$. To generate an artificial sequence, one could down-sample the image to be of size $116411$, $116410$, and so on. At each point in the sequence, one could normalize the point so that its $\ell^2$ norm is $\Vert \mathbf{x}^{(116412)} \Vert$. One could similarly up-sample the image.}

\rtwo{Intuitively, if the up-sampled image does not just insert zeros, as $m$ increases the we expect the ratio $\frac{ |x^{(m)}_i|}{\Vert \mathbf{x}^{(m)} \Vert}$ to decrease because the denominator stays fixed and the numerator gets smaller. In our proof the application of CLT requires $\max_{i=1}^m \frac{ |x^{(m)}_i|}{\Vert \mathbf{x}^{(m)} \Vert}$ to decrease faster than $m^{1/4}$. Hypothesis~\ref{hyp} states this condition precisely.}\\

\begin{hyp}
\label{hyp}
For $\mathbf{x}^{(m)}, \mathbf{y}^{(m)} \in \mathbb{R}^m$, define sequences of inputs $\{ \mathbf{x}^{(m)}\}_{m=2}^\infty$ and  $\{\mathbf{y}^{(m)} \}_{m=2}^\infty$ with fixed $\Vert \mathbf{x}^{(m)} \Vert{=}\Vert \mathbf{x} \Vert$, $\Vert \mathbf{y}^{(m)} \Vert{=}\Vert \mathbf{y} \Vert$, and $\theta_0{=}\cos^{-1} \frac{\mathbf{x}^{(m)} \cdot \mathbf{y}^{(m)}}{\Vert \mathbf{x} \Vert \Vert \mathbf{y} \Vert}$ for all $m$. 

Letting $x^{(m)}_i$ be the $i^{th}$ coordinate of $\mathbf{x}^{(m)}$, assume that $\lim\limits_{m\to \infty} m^{(1/4)} \max_{i=1}^m \frac{ |x^{(m)}_i|}{\Vert \mathbf{x} \Vert}$ and $\lim\limits_{m\to \infty} m^{(1/4)} \max_{i=1}^m \frac{ |y^{(m)}_i|}{\Vert \mathbf{y} \Vert}$ are both $0$.
\end{hyp}

Figures~\ref{fig:asymptote1} and~\ref{fig:asymptote2}  empirically investigate Hypothesis~\ref{hyp} for two datasets, suggesting it makes reasonable assumptions on high dimensional data such as images and audio.

\newcommand{\sizescal}{0.195}
\newcommand{\sizespac}{0.44}
\begin{figure}[!htbp]
\centering
\begin{subfigure}[t]{\sizespac \linewidth}
\includegraphics[scale=0.11]{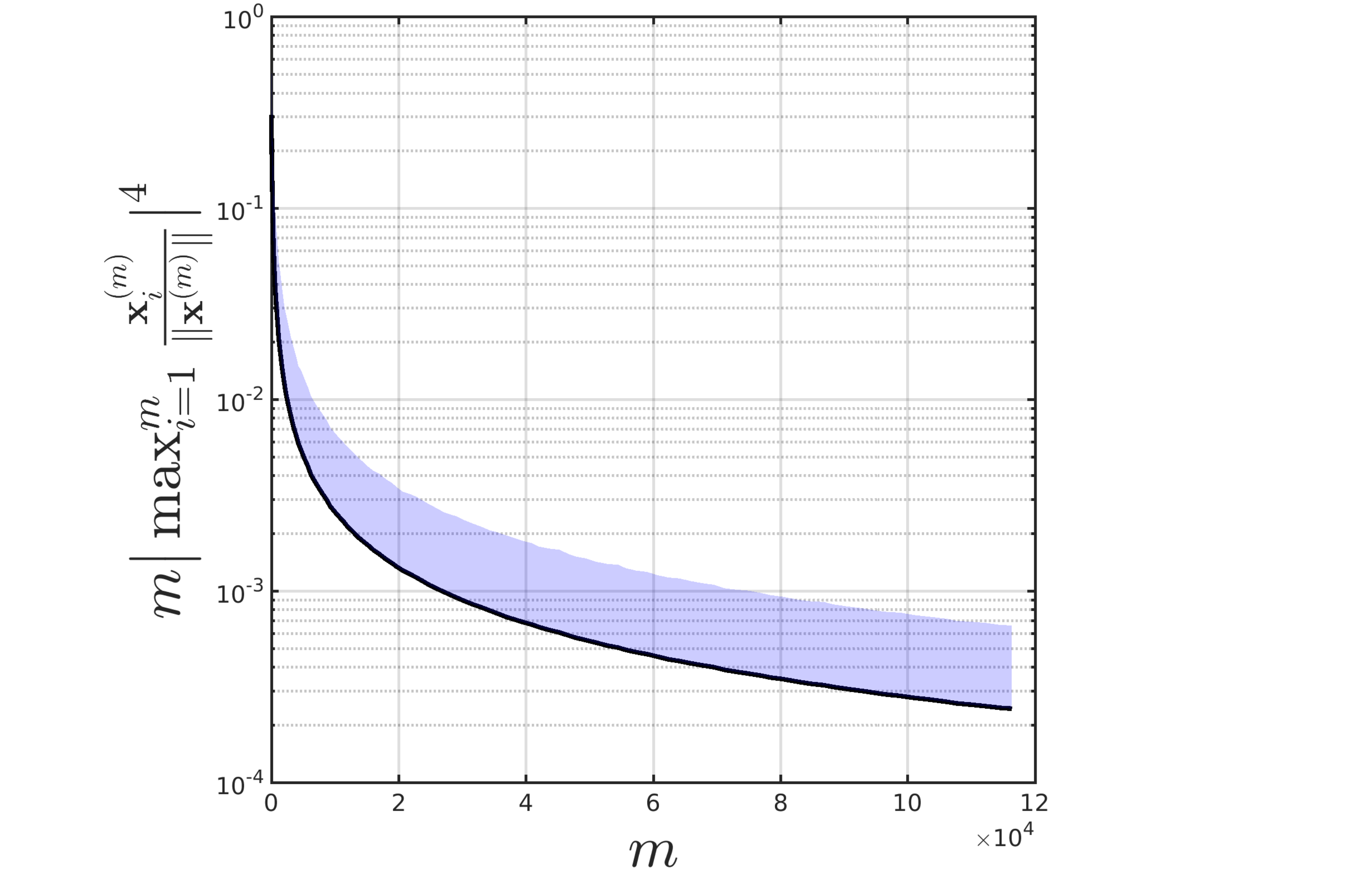}
\end{subfigure}
\begin{subfigure}[t]{\sizespac \linewidth}
\includegraphics[scale=0.11,clip, trim=0cm 0cm 9cm 0cm]{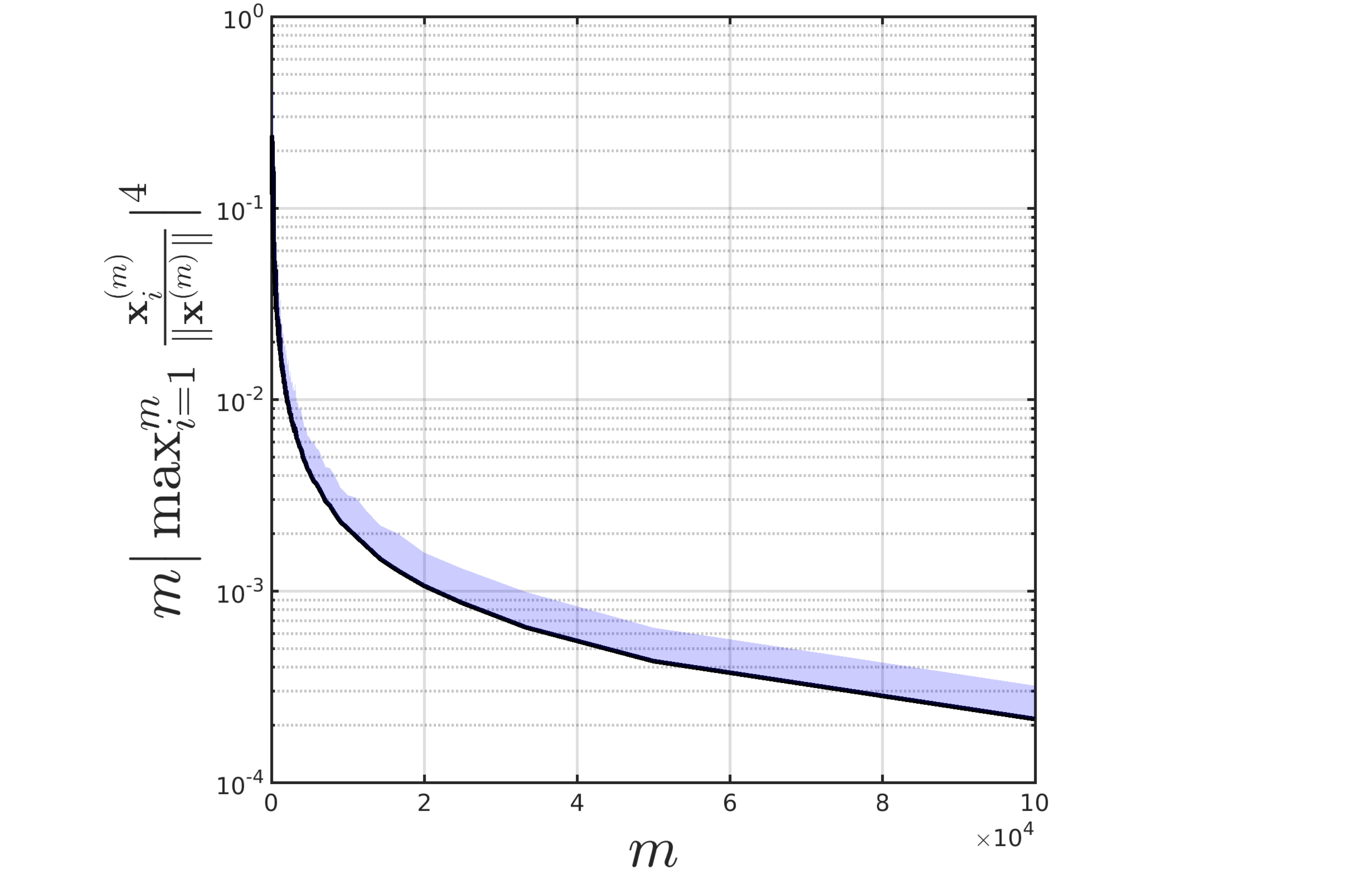}
\end{subfigure}
\caption{The solid line is an average over $1000$ randomly sampled datapoints. The shaded region represents $1$ standard deviation in the worst-case direction. Data is preprocessed so that each dimension is in the range $[0, 255]$. (Left) Aligned and cropped CelebA dataset~\cite{liu2015faceattributes}, with true dimensionality $m=116412$. The images are compressed using Bicubic Interpolation. (Right) CHiME3\_embedded\_et05\_real live speech data from The 4th CHiME Speech Separation and Recognition Challenge~\cite{Vincent2017csl,Barker2017csl}. Each clip is trimmed to a length of $6.25$ seconds and the true sample rate is $16000$ Hz, so the true dimensionality is $m=100000$. Compression is achieved through subsampling by integer factors.} 
\label{fig:asymptote1}
\end{figure}

\begin{theorem}
\label{thm:asymptotic}
Consider an infinitely wide FC layer with almost everywhere continuous activation functions $\sigma$. Suppose the random weights $\mathbf{W}$ come from an IID distribution with PDF $f_m$ such that $\mathbf{E}[W_i]=0$ and $\mathbf{E}|W_i^3|<\infty$. Suppose that the conditions in Hypothesis~\ref{hyp} are satisfied. Then 
$$ \sigma(\mathbf{W}^{(m)} \cdot \mathbf{x}^{(m)})\sigma( \mathbf{W}^{(m)} \cdot \mathbf{y}^{(m)}) \xrightarrow{D} \sigma(Z_1) \sigma(Z_2), $$
where $\xrightarrow{D}$ denotes convergence in distribution and $\mathbf{Z}=(Z_1, Z_2)^T$ is a Gaussian random vector with covariance matrix $ \mathbb{E}[W_i^2]
\begin{bmatrix}
    \Vert \mathbf{x} \Vert^2       & \Vert \mathbf{x} \Vert\Vert \mathbf{y} \Vert\cos\theta_0 \\
    \Vert \mathbf{x} \Vert\Vert \mathbf{y} \Vert\cos\theta_0       & \Vert \mathbf{y} \Vert^2
\end{bmatrix}$ and $\mathbf{0}$ mean. Every $\mathbf{Z}^{(m)}=(\mathbf{W}^{(m)} \cdot \mathbf{x}^{(m)}, \mathbf{W}^{(m)} \cdot \mathbf{y}^{(m)})^T$ has the same mean and covariance matrix as $\mathbf{Z}.$
\end{theorem}
Convergence in distribution is a weak form of convergence, so we cannot expect in general that all kernels should converge asymptotically. For some special cases however, this is indeed possible to show. We first present the ReLU case.
\begin{cor}
\label{cor:relucase}
Let $m$, $\mathbf{W}$, $f_m$, $\mathbf{E}[W_i]$  and $\mathbf{E}|W_i^3|$ be as defined in Theorem~\ref{thm:asymptotic}.  Define the corresponding kernel to be $k^{(m)}_f \big( \mathbf{x}^{(m)},\mathbf{y}^{(m)} \big).$
Consider a second infinitely wide FC layer with $m$ inputs. Suppose the random weights come from a spherical Gaussian with $\mathbb{E}[W_i]=0$ and finite variance $\mathbb{E}[W_i^2]$ with PDF $g_m$. Define the corresponding kernel to be $k^{(m)}_g \big( \mathbf{x}^{(m)},\mathbf{y}^{(m)} \big)$.
Suppose that the conditions in Hypothesis~\ref{hyp} are satisfied and the activation functions are $\sigma(z)= \Theta(z) z$. Then for all $s \geq 2$,
\begin{align*}
&\quad \lim_{m \to \infty} k^{(m)}_f \big( \mathbf{x}^{(m)},\mathbf{y}^{(m)} \big) \\
&= k^{(s)}_g \big( \mathbf{x}^{(s)},\mathbf{y}^{(s)} \big)=\mathbb{E} \big[ \sigma(Z_1) \sigma(Z_2) \big],
 \end{align*}
where $\mathbf{Z}$ is as in Theorem~\ref{thm:asymptotic}. Explicitly, $k^{(m)}_f$ converges to~\eqref{eq:chokernel}.
\end{cor}

\newcommand{\sizescale}{0.195}
\newcommand{\sizespace}{0.245}
\begin{figure*}[t!]
\centering
\begin{subfigure}[t]{\sizespace \linewidth}
\includegraphics[scale= \sizescale ]{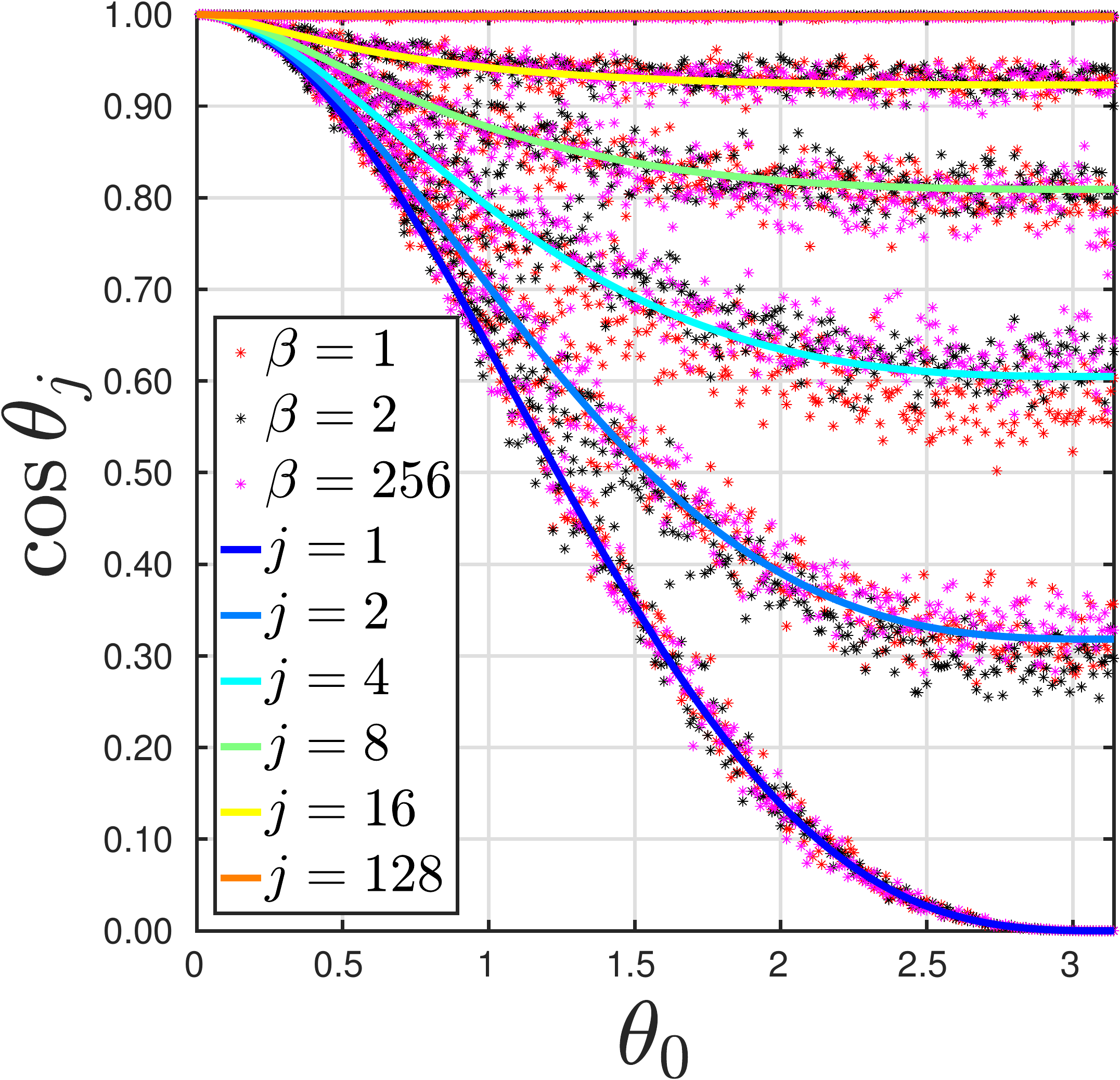} \caption{}
\end{subfigure}
\begin{subfigure}[t]{\sizespace \linewidth}
\includegraphics[scale= \sizescale ]{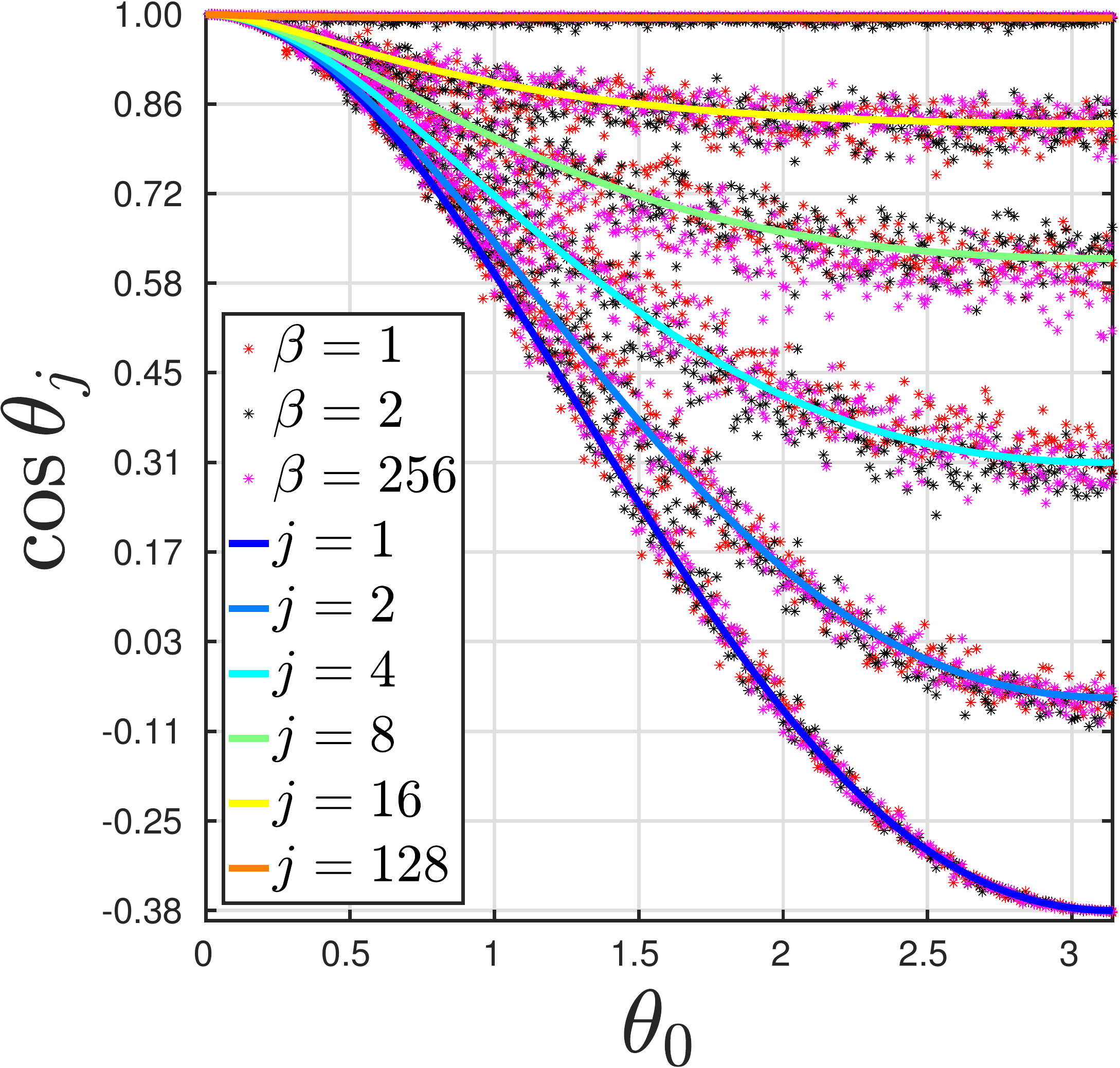} \caption{}
\end{subfigure}
\begin{subfigure}[t]{\sizespace \linewidth}
\includegraphics[scale= \sizescale ]{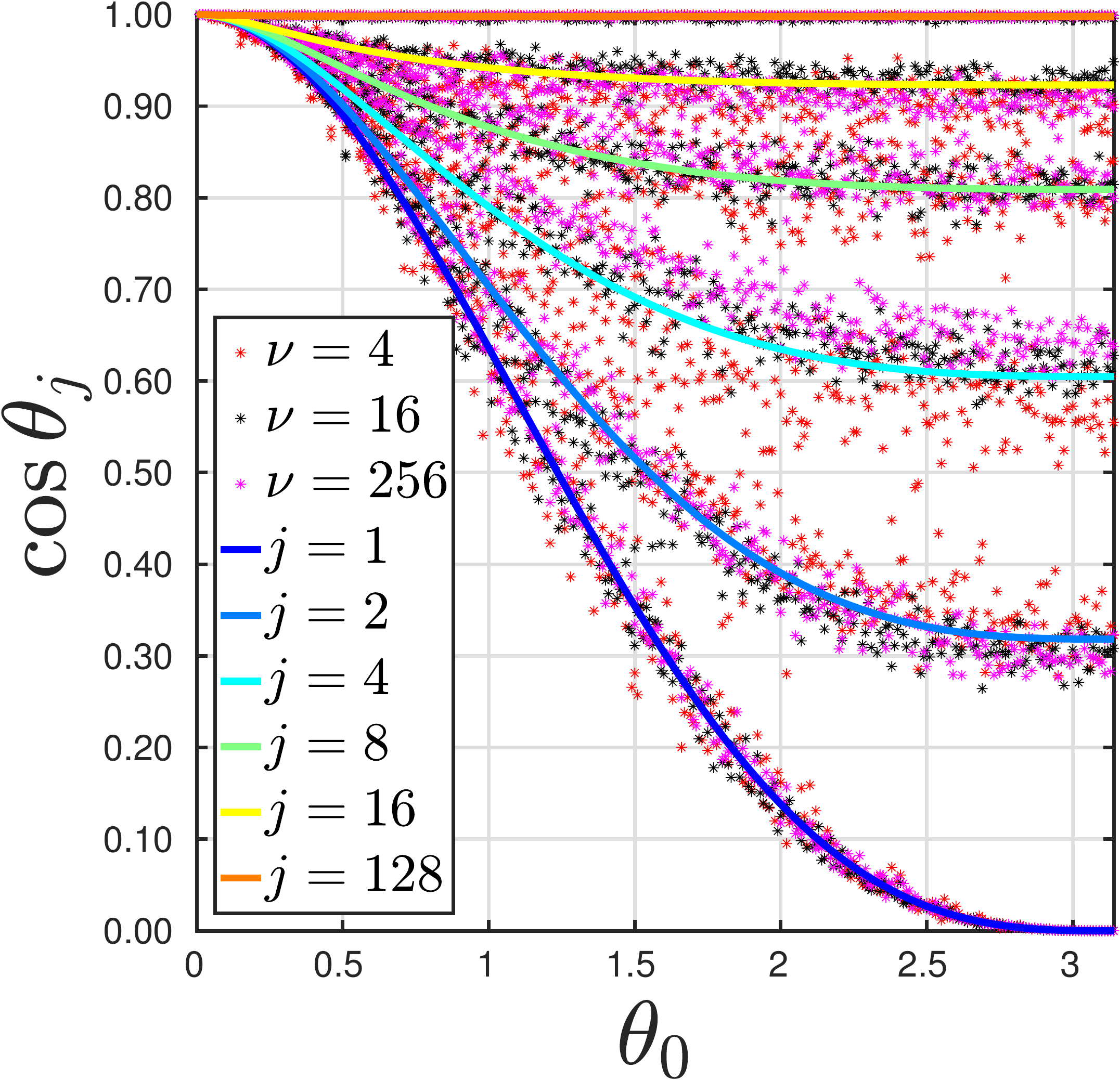} \caption{}
\end{subfigure}
\begin{subfigure}[t]{\sizespace \linewidth}
\includegraphics[scale= \sizescale ]{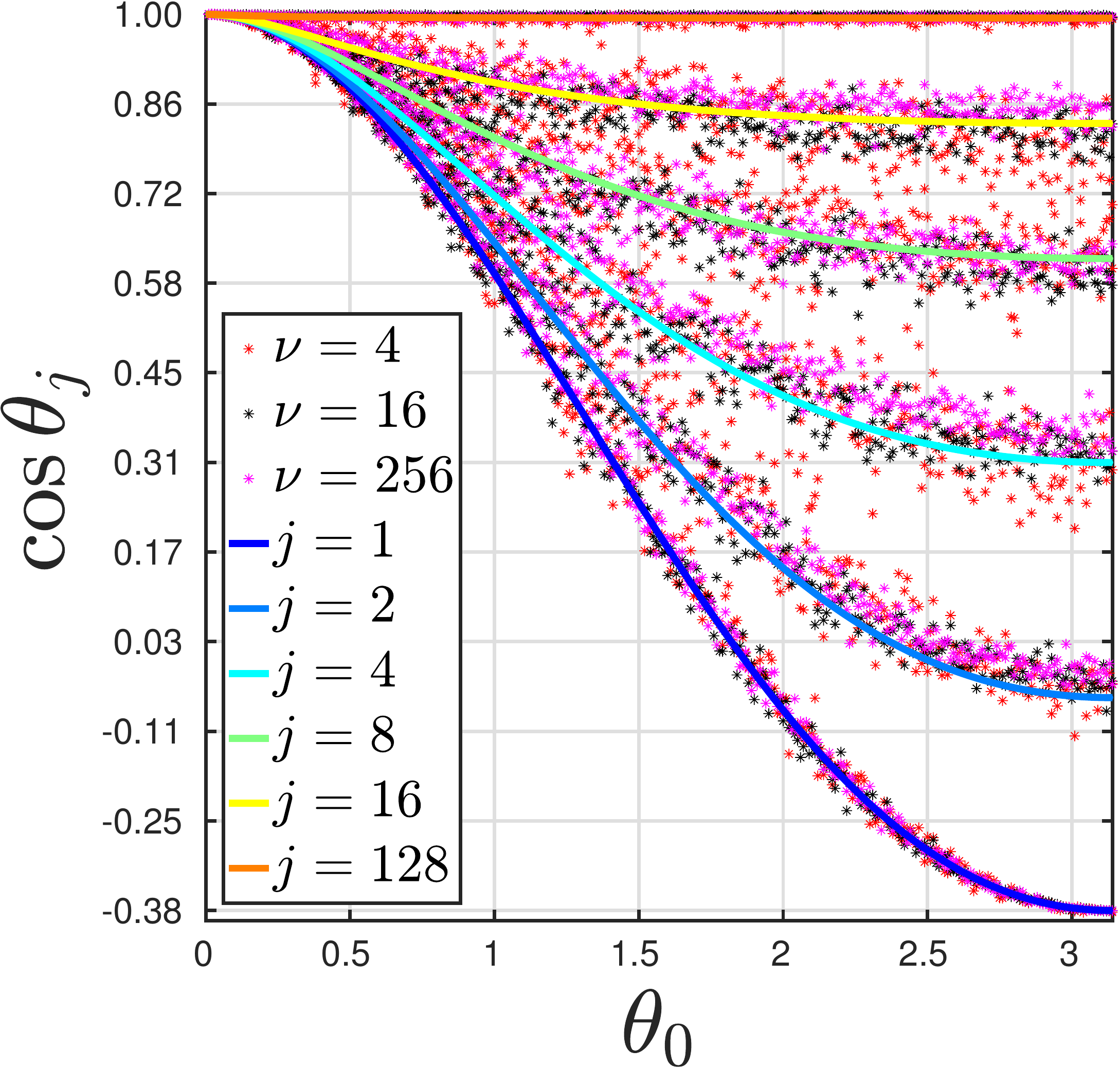} \caption{}
\end{subfigure}
\caption{Theoretical normalized kernel for networks of increasing depth. Empirical samples from a network with between $1$ and $128$ hidden layers, $1000$ hidden neurons in each layer, $m=1000$ and weights coming from different symmetric distributions. The sampling process for each $\theta_0$ is as described in Figure~\ref{fig:arccosker}. The variance is chosen according to~\eqref{eq:lreluinit}. (a) ReLU Activations,~\eqref{eq:gengauss} distribution. (b) LReLU Activations with $a=0.2$,~\eqref{eq:gengauss} distribution. (c) ReLU Activations, t-distribution. (d) LReLU Activations with $a=0.2$, t-distribution.}
\label{fig:composed}
\end{figure*}

The proof is given in Appendix~\ref{app:reluasym}. This implies that the Arc-Cosine Kernel is well approximated by ReLU layers with weights from a wide class of distributions. Similar results hold for other $\sigma$ including the LReLU and ELU~\cite{clevert2015fast}, as shown in the supplementary material.

\section{Empirical Verification of Results}
\label{sec:empir}
%We now empirically verify the theoretical results presented in Section~\ref{sec:theory}.
We empirically verify our results using two families of weight distributions. First, consider the $m$-dimensional t-distribution 
\begin{align*}
f(\mathbf{w}) &= \frac{\Gamma[(\nu + m)/2]}{\Gamma(\nu/2)\nu^{m/2}\pi^{m/2} \sqrt{ |\text{det}(\Sigma)|} } \\
&\quad \Big[ 1 + \frac{1}{\nu}(\mathbf{w}^T\Sigma^{-1}\mathbf{w}) \Big]^{-(\nu+m)/2},
\end{align*}
with degrees of freedom $\nu$ and identity shape matrix $\Sigma= I$. The multivariate t-distribution approaches the multivariate Gaussian as $\nu \to \infty$. Random variables drawn from the multivariate t-distribution are uncorrelated but not independent. This distribution is rotationally-invariant and satisfies the conditions in Propositions~\eqref{cor:kernel} and~\eqref{cor:lrelukernel}. 

Second, consider the multivariate distribution
\begin{equation}
\label{eq:gengauss}
f(\mathbf{w}) = \prod_{i=1}^m \frac{\beta}{2\alpha \Gamma(1/\beta)} e^{-|w_i/\alpha|^\beta},
\end{equation}
which is not rotationally-invariant (except when $\beta = 2,$ which coincides with a Gaussian distribution) but whose random variables are IID and satisfy the conditions in Theorem~\ref{thm:asymptotic}. As $\beta \to \infty$ this distribution converges pointwise to the uniform distribution on $[-\alpha, \alpha]$.

In Figure~\ref{fig:composed}, we empirically verify Propositions~\ref{cor:kernel} and~\ref{cor:lrelukernel}. In the one hidden layer case, the samples follow the blue curve $j=1$, regardless of the specific multivariate t weight distribution which varies with $\nu$. We also observe that the universality of the equivalent kernel appears to hold for the distribution~\eqref{eq:gengauss} regardless of the value of $\beta$, as predicted by theory. We discuss the relevance of the curves $j \neq 1$ in Section~\ref{sec:cor}.

\section{Implications for Practical Networks}
\label{sec:cor}
\subsection{Composed Kernels in Deep Networks}
A recent advancement in understanding the difficulty in training deep neural networks is the identification of the shattered gradients problem~\cite{pmlr-v70-balduzzi17b}. Without skip connections, the gradients of deep networks approach white noise as they are backpropagated through the network, making them difficult to train.

A simple observation that complements this view is obtained through repeated composition of the normalized kernel. As $m \to \infty$, the angle between two inputs in the $j^{th}$ layer of a LReLU network random weights with $\mathbb{E}[W]=0$ and $\mathbb{E}|W^3| < \infty$ approaches $\cos\theta_j = \frac{1}{1+a^2} \Big( \frac{(1-a)^2}{\pi} \big( \sin\theta_{j-1} + (\pi-\theta_{j-1}) \cos\theta_{j-1} \big) + 2a\cos\theta_0 \Big)$.

A result similar to the following is hinted at by Lee et al.~\yrcite{lee2017gp}, citing Schoenholz et al.~\yrcite{schoenholz2016deep}. Their analysis, which considers biases in addition to weights~\cite{poole2016exponential}, yields insights on the trainability of random neural networks that our analysis cannot. However, their argument does not appear to provide a complete formal proof for the case when the activation functions are unbounded, e.g., ReLU. The degeneracy of the composed kernel with more general activation functions is also proved by Daniely~\yrcite{daniely2016toward}, with the assumption that the weights are Gaussian distributed.
\begin{cor}
\label{cor:fixed_point}
\rtwo{The normalized kernel corresponding to LReLU activations converges to a fixed point at $\theta^* = 0$.}
\end{cor}
\begin{proof}
Let $z=\cos\theta_{j-1}$ and define $$T(z){=}\frac{1}{1+a^2}\Big(\frac{(1-a)^2}{\pi} \big( \sqrt{1-z^2} + (\pi - \cos^{-1}z)z \big)+2az \Big).$$ The magnitude of the derivative of $T$ is $\Big|1-\big(\frac{1-a}{1+a} \big)^2\frac{\cos^{-1}z}{\pi}\Big|$ which is bounded above by $1$ on $[-1, 1]$. Therefore, $T$ is a contraction mapping. By Banach's fixed point theorem there exists a \textit{unique} fixed point $z^*=\cos\theta^*$. Set $\theta^*=0$ to verify that $\theta^*=0$ is a solution, and $\theta^*$ is unique.
\end{proof}
\vspace{-0.5em}
Corollary~\ref{cor:fixed_point} implies that for this deep network, the angle between any two signals at a deep layer approaches $0$. No matter what the input is, the kernel ``sees" the same thing after accounting for the scaling induced by the norm of the input. Hence, it becomes increasingly difficult to train deeper networks, as much of the information is lost and the outputs will depend merely on the norm of the inputs; the signals decorrelate as they propagate through the layers. 

At first this may seem counter-intuitive. An appeal to intuition can be made by considering the corresponding linear network with deterministic and equal weight matrices in each layer, which amounts to the celebrated power iteration method. In this case, the repeated application of a matrix transformation $A$ to a vector $v$ converges to the dominant eigenvector (i.e. the eigenvector corresponding to the largest eigenvalue) of $A$.

Figure~\ref{fig:composed} shows that the theoretical normalized kernel for networks of increasing depth closely follows empirical samples from randomly initialized neural networks.

In addition to convergence of direction, by also requiring that $\Vert \mathbf{x} \Vert= \Vert \mathbf{y}\Vert$ it can be shown that after accounting for scaling, the magnitude of the signals converge as the signals propagate through the network. This is analogous to having the dominant eigenvalue equal to $1$ in the power iteration method comparison.
%\begingroup\abovedisplayskip=3pt \belowdisplayskip=4pt \abovedisplayshortskip=3pt \belowdisplayshortskip=4pt
%\allowdisplaybreaks
\begin{cor}
The quantity $\mathbb{E} \Big[ \big( \sigma^{(j)}(\mathbf{x}) - \sigma^{(j)}(\mathbf{y}) \big)^2 \Big]/\mathbb{E}[\sigma^{(j)}(\mathbf{x})^2]$ in a $j$-layer random (L)ReLU network of infinite width with random uncorrelated and identically distributed rotationally-invariant weights with $\Vert \mathbf{x} \Vert{=}\Vert \mathbf{y} \Vert$ approaches $0$ as $j \to \infty$.
\end{cor} 
\begin{proof}
Denote the output of one neuron in the $j^{th}$ layer of a network $\sigma(\bm{W^{(1)}} \cdot \bm{\sigma}(...\bm{\sigma}(W^{(j})\mathbf{x}))$ by $\sigma^{(j)}(\mathbf{x})$ and let $k_j$ be the kernel for the $j$-layer network. Then
\begin{align*}
\mathbb{E} &\Big[ \big( \sigma^{(j)}(\mathbf{x}) - \sigma^{(j)}(\mathbf{y}) \big)^2 \Big]/\mathbb{E}[\sigma^{(j)}(\mathbf{x})^2]\\
&= \big(k_j(\mathbf{x}, \mathbf{x})- 2 k_j(\mathbf{x}, \mathbf{y}) +  k_j(\mathbf{y}, \mathbf{y}) \big)/k_j(\mathbf{x}, \mathbf{x}), \\
&= 2 - 2\cos\theta_j
\end{align*}
which approaches $0$ as $j \to \infty$.
\end{proof}
\vspace{-0.5em}
%\endgroup
Contrary to the shattered gradients analysis, which applies to gradient based optimizers, our analysis relates to any optimizers that initialize weights from some distribution satisfying conditions in Proposition~\ref{cor:lrelukernel} or Corollary~\ref{cor:relucase}. Since information is lost during signal propagation, the network's output shares little information with the input. An optimizer that tries to relate inputs, outputs and weights through a suitable cost function will be ``blind" to relationships between inputs and outputs.

Our results can be used to argue against the utility of controversial Extreme Learning Machines (ELM)~\cite{huang2004extreme}, which randomly initialize hidden layers from symmetric distributions and only learn the weights in the final layer. A single layer ELM can be replaced by kernel ridge regression using the equivalent kernel. Furthermore, a Multi-Layer ELM~\cite{tang2016extreme} with (L)ReLU activations utilizes a pathological kernel as shown in Figure~\ref{fig:composed}. It should be noted that ELM bears resemblance to early works~\cite{schmidt1992feedforward, pao1994learning}.

\newcommand{\histscale}{0.18}
\newcommand{\histheight}{3.2}
\begin{figure}[!t]
\centering
\includegraphics[height=\histheight cm]{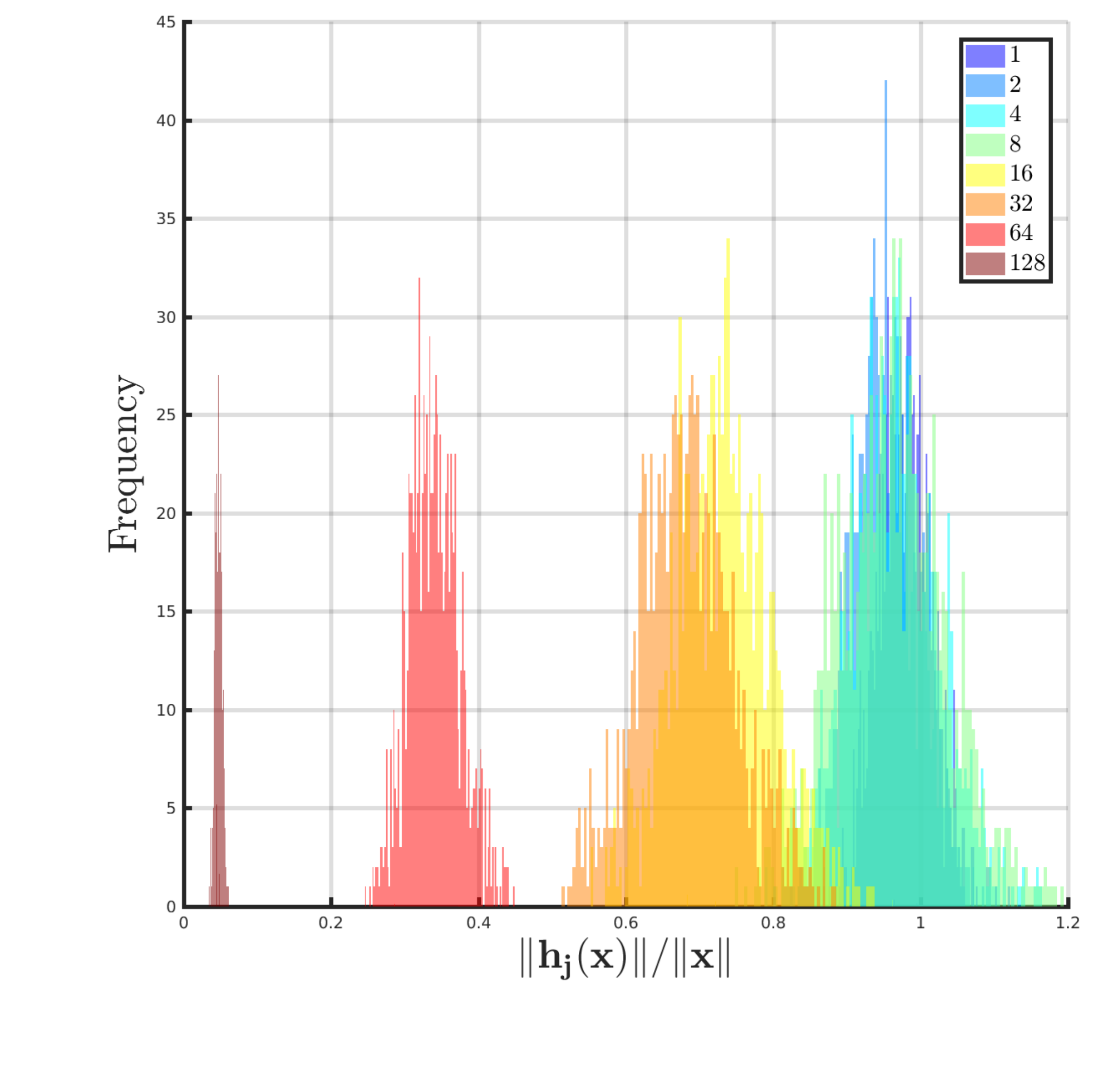}
\includegraphics[height=\histheight cm]{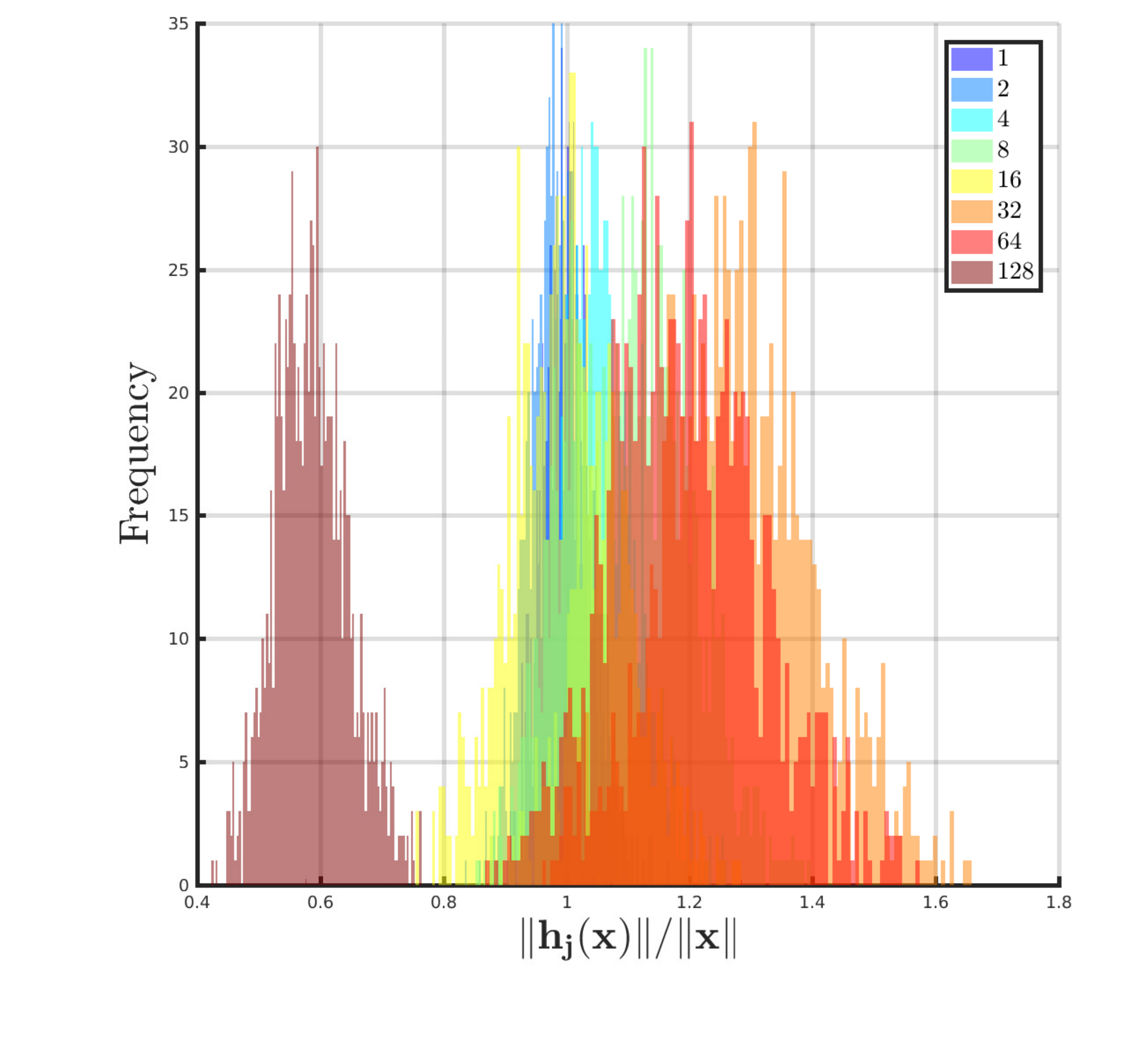}
\caption{Histograms showing the ratio of the norm of signals in layer $j$ to the norm of the input signals. Each histogram contains $1000$ data points randomly sampled from a unit Gaussian distribution. The network tested has $1000$ inputs, $1000$ neurons in each layer, and LReLU activations with $a=0.2$. The legend indicates the number of layers in the network. The weights are randomly initialized from a Gaussian distribution. (Left) Weights initialized according to the method of He et al.~\yrcite{he2015delving}. (Right) Weights initialized according to~\eqref{eq:lreluinit}.} \label{fig:norms} \end{figure}
\subsection{Initialization}
\iffalse
\begin{cor}
\label{cor:normsReLU}
Consider a hidden layer containing LReLU activations with gradient parameter $a \in (0,1)$. Suppose the weights are uncorrelated and identically distributed with a rotationally-invariant PDF. To approximately preserve the $\ell_2$ norm from the input to the output, we have
\begin{equation}
\label{eq:lreluinit}
\sqrt{\mathbb{E}[W^2]}=\sqrt{\frac{2}{\Big( 1+a^2 \Big) n}},
\end{equation}
where $n$ is the number of neurons in the hidden layer.
\end{cor}
\begin{proof}
The $\ell_2$-norm of the signal in the hidden layer for an input $\mathbf{x}$ is approximately $\Vert \mathbf{h}(\mathbf{x}) \Vert = \sqrt{k(\mathbf{x}, \mathbf{x})n},$ as evident in definition~\eqref{kernel}. Letting $\theta_0=0$ in~\eqref{eq:kernel}, we have
$\Vert \mathbf{h}(\mathbf{x}) \Vert = \Vert \mathbf{x} \Vert \sqrt{ \frac{n\mathbb{E}[W_i^2](1+a^2)}{2} }.$
Setting $\Vert \mathbf{h}(\mathbf{x}) \Vert = \Vert \mathbf{x} \Vert$ gives the desired result.
\end{proof}
\vspace{-0.5em}
\fi
\rtwo{Suppose we wish to approximately preserve the $\ell^2$ norm from the input to hidden layer. By comparing~\eqref{eq:sum} and~\eqref{kernel}, we approximately have $\Vert \mathbf{h}(\mathbf{x}) \Vert \approx \sqrt{k(\mathbf{x}, \mathbf{x})n}$. Letting $\theta_0=0$ in~\eqref{eq:kernel}, we have
$\Vert \mathbf{h}(\mathbf{x}) \Vert = \Vert \mathbf{x} \Vert \sqrt{ \frac{n\mathbb{E}[W_i^2](1+a^2)}{2} }.$ Setting $\Vert \mathbf{h}(\mathbf{x}) \Vert = \Vert \mathbf{x} \Vert$,}
\begin{equation}
\label{eq:lreluinit}
\sqrt{\mathbb{E}[W_i^2]}=\sqrt{\frac{2}{\Big( 1+a^2 \Big) n}}.
\end{equation}
This applies whenever the conditions in Proposition~\ref{cor:lrelukernel} or Corollary~\ref{cor:lrelucase} are satisfied. This agrees with the well-known case when the elements of $\mathbf{W}$ are IID~\cite{he2015delving} and $a=0$. \iffalse However, there is no need to restrict $\mathbf{W}$ to contain IID elements. Even if the weights are IID, the approximation of equivalent kernel of a ReLU network by the Arc-Cosine kernel holds universally. When the weights come from a Laplace distribution $f(w)=\frac{1}{2b}e^{-|x|/b},$ choosing
$b=\sqrt{\frac{1}{n(1+a^2)}}$
preserves the signal norms.
When the weights come from a uniform distribution on $[-a, a],$ choosing
$a=\sqrt{\frac{6}{n(1+a^2)}}$
preserves the signal norms.\fi For small values of $a$, ~\eqref{eq:lreluinit} is well approximated by the known result~\cite{he2015delving}. For larger values of $a$, this approximation breaks down, as shown in Figure~\ref{fig:norms}.

An alternative approach to weight initialization is the data-driven approach~\cite{mishkin2015all}, which can be applied to more complicated network structures such as convolutional and max-pooling layers commonly used in practice. As parameter distributions change during training, batch normalization inserts layers with learnable scaling and centering parameters at the cost of increased computation and complexity~\cite{ioffe2015batch}.

\section{Conclusion}
We have considered universal properties of MLPs with weights coming from a large class of distributions. We have theoretically and empirically shown that the equivalent kernel for networks with an infinite number of hidden ReLU neurons and all rotationally-invariant weight distributions is the Arc-Cosine Kernel. The CLT can be applied to approximate the kernel for high dimensional input data. When the activations are LReLUs, the equivalent kernel has a similar form. The kernel converges to a fixed point, showing that information is lost as signals propagate through the network.

One avenue for future work is to study the equivalent kernel for different activation functions, noting that some activations such as the ELU may not be expressible in a closed form (we do show in the supplementary material however, that the ELU does have an asymptotically universal kernel).

Since wide networks with centered weight distributions have approximately the same equivalent kernel, powerful trained deep and wide MLPs with (L)ReLU activations should have asymmetric, non-zero mean, non-IID parameter distributions. \rone{Future work may consider analyzing the equivalent kernels of trained networks and more complicated architectures. We should not expect that $k(\mathbf{x}, \mathbf{y})$ may be expressed neatly as $k(\theta_0)$ in these cases. This work is a crucial first step in identifying invariant properties in neural networks and sets a foundation from which we hope to expand in future.}

\appendix 
\section{Proof of Proposition~\ref{lemma}}
\label{app:lemma}
\begingroup\abovedisplayskip=3pt \belowdisplayskip=4pt \abovedisplayshortskip=3pt \belowdisplayshortskip=4pt
\begin{proof}
The kernel with weight PDF $f(\bm{\omega})$ and ReLU $\sigma$ is
$$k(\mathbf{x}, \mathbf{y}) = \int_{\mathbb{R}^m}\Theta(\bm{\omega}\cdot\mathbf{x}) \Theta(\bm{\omega}\cdot\mathbf{y}) (\bm{\omega} \cdot \mathbf{x}) (\bm{\omega} \cdot \mathbf{y}) f(\bm{\omega}) \,d\bm{\omega}.$$ 
Let $\theta_0$ be the angle between $\mathbf{x}$ and $\mathbf{y}$. Define $\mathbf{u} = (\Vert \mathbf{x} \Vert, 0, ..., 0)^T$ and $\mathbf{v} = (\Vert \mathbf{y} \Vert \cos\theta_0, \Vert \mathbf{y} \Vert \sin\theta_0, 0, ..., 0)^T$ with $\mathbf{u}, \mathbf{v} \in \mathbb{R}^m$. Following Cho \& Saul~\yrcite{cho2009kernel}, there exists some $m \times m$ rotation matrix $R$ such that $\mathbf{x} = R\mathbf{u}$ and $\mathbf{y} = R \mathbf{v}$. We have
\begin{align*}
k(\mathbf{x}, \mathbf{y}) &= \int_{\mathbb{R}^m}\Theta(\bm{\omega}\cdot R \mathbf{u}) \Theta(\bm{\omega}\cdot R \mathbf{v}) (\bm{\omega} \cdot R \mathbf{u}) (\bm{\omega} \cdot R \mathbf{v}) \\
& \quad f(\bm{\omega}) \,d\bm{\omega}.
\end{align*}
Let $\bm{\omega}=R \mathbf{w}$ and note that the dot product is invariant under rotations and the determinant of the Jacobian of the transformation is $1$ since $R$ is orthogonal. We have
\begin{align*}
k(\mathbf{x}, \mathbf{y}) &= \int_{\mathbb{R}^m}\Theta(\mathbf{w}\cdot  \mathbf{u}) \Theta(\mathbf{w}\cdot \mathbf{v}) (\mathbf{w} \cdot  \mathbf{u}) (\mathbf{w} \cdot  \mathbf{v}) \\
&\quad f(R \mathbf{w}) \,d\mathbf{w}, \\
&= \int_{\mathbb{R}^m}\Theta(\Vert \mathbf{x} \Vert w_1 ) \Theta(\Vert \mathbf{y} \Vert (w_1 \cos\theta_0 + w_2 \sin\theta_0)) \\
&\quad w_1 (w_1 \cos\theta_0 + w_2 \sin\theta_0)f( \mathbf{w}) \,d\mathbf{w} \Vert \mathbf{x} \Vert \Vert \mathbf{y} \Vert \numberthis \label{eq:ker}.
\end{align*}
One may view the integrand as a functional acting on test functions of $\theta_0$. Denote the set of infinitely differentiable test functions on $(0, \pi)$ by $C_c^\infty(0,\pi)$. The linear functional acting over $C_c^\infty(0,\pi)$ is a Generalized Function and we may take distributional derivatives under the integral by Theorem 7.40 of Jones~\yrcite{jones_1982}. Differentiating twice,
\begin{align*}
k'' &+ k \\
&= \int_{\mathbb{R}^m}\Theta(w_1 ) w_1 (-w_1 \sin\theta_0 + w_2 \cos\theta_0)^2 \\
&\quad \delta \big( w_1 \cos\theta_0 + w_2 \sin\theta_0 \big) f( \mathbf{w}) \,d\mathbf{w} \Vert \mathbf{x} \Vert \Vert \mathbf{y} \Vert, \\
&= \int_{\mathbb{R}^{m-1}} f \Big(  (s \sin\theta_0, -s\cos\theta_0, w_3, ..., w_m)^T \Big) \\
&\quad \Theta(s) s^3 \,ds\,dw_3\,dw_4...\,dw_m \Vert \mathbf{x} \Vert \Vert \mathbf{y} \Vert \sin\theta_0.
\end{align*}
The initial condition $k(\pi)=0$ is obtained by putting $\theta_0=\pi$ in~\eqref{eq:ker} and noting that the resulting integrand contains a factor of $\Theta(w_1)\Theta(-w_1)w_1$ which is $0$ everywhere. Similarly, the integrand of $k'(\pi)$ contains a factor of $\Theta(w_2)\Theta(-w_2)w_2$. 

The ODE is meant in a distributional sense, that 
$$\int_0^\pi \psi(\theta_0) \big( k''(\theta_0) + k(\theta_0) - F(\theta_0) \big)\,d\theta_0=0$$ 
$\forall \psi \in C_c\infty(0, \pi)$, where $k$ is a distribution with a distributional second derivative $k''$.
\end{proof}

\section{Proof of Proposition~\ref{theorem}}
\label{app:theorem}
\begin{proof}
Denote the marginal PDF of the first two coordinates of $\mathbf{W}$ by $f_{12}.$ Due to the rotational invariance of $f$, $f(O\mathbf{x})= f(\Vert \mathbf{x} \Vert) = f(\mathbf{x})$ for any orthogonal matrix $O$. So
\begin{align*}
F(\theta_0) &= \int_{\mathbb{R}^{m-1}} f\big( (s \sin\theta_0, -s\cos\theta_0, w_3, ..., w_m)^T\big) \\
&\quad  \sin\theta_0 \Theta(s) s^3 \,ds\,dw_3,...\,dw_m \Vert \mathbf{x} \Vert \Vert \mathbf{y} \Vert, \\
&= \sin\theta_0 \int_{\mathbb{R}} \Theta(s) s^3 f_{12} \big( (s, 0,)^T \big) \,ds \Vert \mathbf{x} \Vert \Vert \mathbf{y} \Vert, \\
&= K \sin\theta_0, \quad  K \in (0, \infty].
\end{align*}
It remains to check that $K < \infty$. $F$ is integrable since
\begin{align*}
&\quad\int_{\mathbb{R}^2}\int_0^\pi \Theta(w_1) w_1 (-w_1\sin\theta_0+w_2\cos\theta_0)^2 \\
&\quad\delta(w_1\cos\theta_0+w_2\sin\theta_0)f_{12}(w_1, w_2) d\theta_0 dw_1 dw_2 \\
&= \int_{\mathbb{R}^2}\Theta(w_1)w_1 \big| (w_1^2 + w_2^2)^{1/2} \big| f_{12}(w_1, w_2) dw_1 dw_2, \\
&\leq \sqrt{\mathbb{E} \big[ \Theta^2(W_1)W_1^2 \big]} \sqrt{ \mathbb{E} \big[W_1^2 + W_2^2]} < \infty.
\end{align*}
Therefore, $F$ is finite almost everywhere. This is only true if $K <\infty$. $k''=F-k$ must be a function, so the distributional and classical derivatives coincide.
\end{proof}

\section{Proof of Theorem~\ref{thm:asymptotic}}
\label{app:asymptotic}
\begin{proof}
There exist some orthonormal $\mathbf{R_1}, \mathbf{R_2}{\in}\mathbb{R}^m$ such that $\mathbf{y}^{(m)}{=}\Vert \mathbf{y}^{(m)} \Vert (\mathbf{R_1} \cos\theta_0 + \mathbf{R_2}\sin\theta_0)$ and $\mathbf{x}^{(m)}=\Vert \mathbf{x}^{(m)} \Vert \mathbf{R_1}$. We would like to examine the asymptotic distribution of $\sigma \big(\Vert \mathbf{y}^{(m)} \Vert \mathbf{W}^{(m)}{\cdot}\big( \mathbf{R_1}\cos\theta_0{+}\mathbf{R_2}\sin\theta_0 \big) \big) \\ \sigma \big(\Vert \mathbf{x}^{(m)} \Vert \mathbf{W}^{(m)}{\cdot}\mathbf{R_1}\big).$

Let $U_1^{(m)}{=}\mathbf{W} \cdot \mathbf{R_1}\cos\theta_0 + \mathbf{W} \cdot \mathbf{R_2}\sin\theta_0$ and $\\U_2^{(m)}={-}\mathbf{W}{\cdot}\mathbf{R_1}\sin\theta_0{+}\mathbf{W}{\cdot}\mathbf{R_2}\cos\theta_0$. Note that $\mathbb{E}[U_1^{(m)2}]{=}\mathbb{E}[U_2^{(m)2}]{=}\mathbb{E}[W_i^2]$ and $\mathbb{E}[U^{(m)}_1]{=}\mathbb{E}[U^{(m)}_2]{=}0$. Also note that $U^{(m)}_1$ and $U^{(m)}_2$ are uncorrelated since $\mathbb{E}[U^{(m)}_1U^{(m)}_2]=\mathbb{E} \Big[ (\mathbf{W} \cdot \mathbf{R_1})(\mathbf{W} \cdot \mathbf{R_2})(\cos^2\theta_0 + \sin^2\theta_0)-\cos\theta_0\sin\theta_0\big( (\mathbf{W} \cdot \mathbf{R_1})^2 - (\mathbf{W} \cdot \mathbf{R_2})^2 \big) \Big]=0$.

Let $M_k= \mathbb{E} \big| W_i^k \big|$, $\mathbf{U}^{(m)}=(U_1, U_2)^T$, $I$ be the $2\times2$ identity matrix and $\mathbf{Q} \sim N\big(\mathbf{0}, M_2 I\big)$. Then for any convex set $S \in \mathbb{R}^2$ and some $C \in \mathbb{R}$, by the Berry-Esseen Theorem, \\ $\big| \mathbb{P}[\mathbf{U} \in S] - \mathbb{P}[\mathbf{Q} \in S] \big|^2 \leq C \gamma^2$ where $\gamma^2$ is given by
\begin{align*}
&\Big( \sum_{j=1}^m \mathbb{E} \Big\Vert M_2^{\frac{-1}{2}} W_i \, I \begin{pmatrix}
R_{1j}\cos\theta_0+R_{2j}\sin\theta_0 \\
-R_{1j}\sin\theta_0 + R_{2j}\cos\theta_0
\end{pmatrix}\Big\Vert^3 \Big)^2, \\
&=\Big( M_2^{\frac{-3}{2}} M_3\sum_{j=1}^m \mathbb{E} \Big\Vert \begin{pmatrix}
R_{1j}\cos\theta_0+R_{2j}\sin\theta_0 \\
-R_{1j}\sin\theta_0 + R_{2j}\cos\theta_0
\end{pmatrix}\Big\Vert^3 \Big)^2, \\
&= \Big(  M_2^{\frac{-3}{2}} M_3 \sum_{j=1}^m \Big| R_{1j}^2+R_{2j}^2 \Big|^{(3/2)} \Big)^2, \\
&\leq  M_2^{-3} M_3^2 m\sum_{j=1}^m \Big| R_{1j}^2+R_{2j}^2 \Big|^3,\\
&= M_2^{-3} M_3^2 m\sum_{j=1}^m \Big| R_{1j}^6+3R_{1j}^4 R_{2j}^2+3R_{1j}^2 R_{2j}^4 + R_{2j}^6 \Big|,\\
&\leq M_2^{-3} M_3^2 m \Big( 4\max_{k=1}^mR_{1k}^4 + 4 \max_{k=1}^m R_{2k}^4\Big).
\end{align*}
The last line is due to the fact that 
\begin{align*}
&\sum_{j=1}^m \Big| R_{1j}^6+3R_{1j}^4 R_{2j}^2 \Big| \leq \max_{k=1}^mR_{1k}^4 \big( \sum_{j=1}^m R_{1j}^2+3R_{2j}^2 \big).
\end{align*}
\iffalse
\begin{align*} &\quad \big| \mathbb{P}[\mathbf{U} \in S] - \mathbb{P}[\mathbf{Q} \in S] \big|^2 \\
&\leq \Big( C  \sum_{j=1}^m \mathbb{E} \Big[ \big\Vert \mathbb{E}[W_i^2]^{-1/2}I \\
&\quad W_i \begin{pmatrix}
R_{1j}\cos\theta_0+R_{2j}\sin\theta_0 \\
-R_{1j}\sin\theta_0 + R_{2j}\cos\theta_0
\end{pmatrix}\big\Vert^3 \Big] \Big)^2, \\
&= \Big( C \mathbb{E}\Big[W_i^2\Big]^{(-3/2)} \mathbb{E}\Big|W_i^3\Big| \sum_{j=1}^m \Big( R_{1j}^2+R_{2j}^2 \Big)^{(3/2)} \Big)^2, \\
&\leq C^2 \mathbb{E}\Big[W_i^2\Big]^{-3} \mathbb{E}\Big|W_i^3\Big|^2 m\sum_{j=1}^m \Big( R_{1j}^2+R_{2j}^2 \Big)^3,\\
&\leq C^2 \mathbb{E}\Big[W_i^2\Big]^{-3} \mathbb{E}\Big|W_i^3\Big|^2 m \Big( 4 \max_{j=1}^mR_{1j}^4 + 4\max_{j=1}^m R_{2j}^4\Big).
\end{align*}
\fi
Now $R_{1k} = \frac{x_k}{\Vert \mathbf{x} \Vert}$ and $R_{2k}=\frac{1}{\sin\theta_0}\Big(\frac{y_k}{\Vert \mathbf{y} \Vert} - \frac{x_k}{\Vert \mathbf{x} \Vert} \cos\theta_0 \Big)$, so if $\theta_0 \neq 0, \pi$ by Hypothesis~\ref{hyp} $\mathbf{U}^{(m)}$ converges \emph{in distribution} to the bivariate spherical Gaussian with variance $\mathbb{E}[W_i^2]$. Then the random vector $\mathbf{Z}^{(m)}=( Z_1^{(m)},Z_2^{(m)})^T= \big( \Vert \mathbf{x} \Vert \mathbf{W} \cdot \mathbf{R_1}, \Vert \mathbf{y} \Vert (\mathbf{W} \cdot \mathbf{R_1}\cos\theta_0 + \mathbf{W} \cdot \mathbf{R_2}\sin\theta_0) \big)^T=\big(\Vert \mathbf{x} \Vert (U_1\cos\theta_0 - U_2\sin\theta_0), \Vert \mathbf{y} \Vert U_1 \big)^T$ converges \emph{in distribution} to the bivariate Gaussian random variable with covariance matrix 
$ \mathbb{E}[W_i^2]
\begin{bmatrix}
    \Vert \mathbf{x} \Vert^2       & \Vert \mathbf{x} \Vert\Vert \mathbf{y} \Vert\cos\theta_0 \\
    \Vert \mathbf{x} \Vert\Vert \mathbf{y} \Vert\cos\theta_0       & \Vert \mathbf{y} \Vert^2
\end{bmatrix}$. Since $\sigma$ is continuous almost everywhere, by the Continuous Mapping Theorem, 
$$ \sigma(\mathbf{W}^{(m)} \cdot \mathbf{x}^{(m)})\sigma( \mathbf{W}^{(m)} \cdot \mathbf{y}^{(m)}) \xrightarrow{D} \sigma(Z_1) \sigma(Z_2).$$
If $\theta_0=0$ or $\theta_0=\pi$, we may treat $\mathbf{R_2}$ as $\mathbf{0}$ and the above still holds.
\end{proof}

\section{Proof of Corollary~\ref{cor:relucase}}
\label{app:reluasym}
\begin{proof}
We have $\lim_{m \to \infty} k^{(m)}_f \big( \mathbf{x}^{(m)},\mathbf{y}^{(m)} \big) = \lim_{m \to \infty} \mathbb{E} \big[ \sigma(Z_1^{(m)}) \sigma(Z_2^{(m)}) \big]$ and would like to bring the limit inside the expected value. By Theorem~\ref{thm:asymptotic} and Theorem 25.12 of Billingsley~\yrcite{Billingsley}, it suffices to show that $\sigma(Z_1^{(m)})\sigma(Z_2^{(m)})$ is uniformly integrable. Define $h$ to be the joint PDF of $\mathbf{Z}^{(m)}$. We have
\begin{align*}
&\quad \lim_{\alpha \to \infty} \int_{| \sigma(z_1) \sigma(z_2)| > \alpha} |\sigma(z_1) \sigma(z_2) | h(z_1, z_2) \, dz_1 dz_2 \\
&= \lim_{\alpha \to \infty} \int_{| \Theta(z_1) \Theta(z_2) z_1 z_2| > \alpha} | \Theta(z_1) \Theta(z_2) z_1 z_2 | h(z_1, z_2) \\
&\quad dz_1 dz_2,
\end{align*}
but the integrand is $0$ whenever $z_1 \leq 0$ \emph{or} $z_2 \leq 0$. So 
\begin{align*}
&\quad\int_{| \sigma(z_1) \sigma(z_2)| > \alpha} |\sigma(z_1) \sigma(z_2) | h(z_1, z_2) \, dz_1 dz_2 \\
&= \int_{ \mathbb{R}^2 } z_1 z_2 \Theta(z_1z_2 - \alpha) \Theta(z_1) \Theta(z_2) h(z_1, z_2) \, dz_1 dz_2.
\end{align*}
We may raise the Heaviside functions to any power without changing the value of the integral. Squaring the Heaviside functions and applying H\"{o}lder's inequality, we have
\begin{align*}
&\Big( \int_{ \mathbb{R}^2 } z_1 z_2 \Theta^2(z_1z_2 - \alpha) \Theta^{2}(z_1) \Theta^{2}(z_2) h(z_1, z_2) dz_1 dz_2 \Big)^2 \\
&\leq\mathbb{E}[z_1^2 \Theta(z_1z_2 - \alpha) \Theta(z_1) \Theta(z_2) ] \\
&\quad\mathbb{E}[z_2^2 \Theta(z_1z_2 - \alpha) \Theta(z_1) \Theta(z_2) ].
\end{align*}
Examining the first of these factors, 
\begin{align*}
&\quad \int_0^\infty \int_{\alpha/z_1}^\infty z_1^2 h(z_1, z_2) \,dz_2 dz_1, \\
&= \int_0^\infty z_1^2 \int_{\alpha/z_1}^\infty h(z_1, z_2) \,dz_2 dz_1. 
\end{align*}
Now let $g_\alpha(z_1) = \int_{\alpha/z_1}^\infty h(z_1, z_2) \,dz_2$. $g_\alpha(z_1)z_1^2$ is monotonically pointwise non-increasing to $0$ in $\alpha$ for all $z_1 > 0$ and $\int z_1^2 g_0(z_1) dz_1 \leq \mathbb{E}[Z_1^2] < \infty $ . By the Monotone Convergence Theorem $\lim_{\alpha \to \infty} \mathbb{E}[z_1^2 \Theta(z_1z_2 - \alpha) \Theta(z_1) ] =0$. The second factor has the same limit, so the limit of the right hand side of H\"{o}lder's inequality is $0$.
\end{proof}

%\endgroup
\newpage
\section*{Acknowledgements}
We thank the anonymous reviewers for directing us toward relevant work and providing helpful recommendations regarding the presentation of the paper. Farbod Roosta-Khorasani gratefully acknowledges the support from the Australian Research Council through a Discovery Early Career Researcher Award (DE180100923). Russell Tsuchida's attendance at the conference was made possible by an ICML travel award.
\bibliography{kernelnn}

\begin{thebibliography}{48}
\providecommand{\natexlab}[1]{#1}
\providecommand{\url}[1]{\texttt{#1}}
\expandafter\ifx\csname urlstyle\endcsname\relax
  \providecommand{\doi}[1]{doi: #1}\else
  \providecommand{\doi}{doi: \begingroup \urlstyle{rm}\Url}\fi

\bibitem[Bach(2017{\natexlab{a}})]{JMLR:v18:14-546}
Bach, F.
\newblock Breaking the curse of dimensionality with convex neural networks.
\newblock \emph{Journal of Machine Learning Research}, 18\penalty0
  (19):\penalty0 1--53, 2017{\natexlab{a}}.

\bibitem[Bach(2017{\natexlab{b}})]{bach2017equivalence}
Bach, F.
\newblock On the equivalence between kernel quadrature rules and random feature
  expansions.
\newblock \emph{Journal of Machine Learning Research}, 18\penalty0
  (21):\penalty0 1--38, 2017{\natexlab{b}}.

\bibitem[Balduzzi et~al.(2017)Balduzzi, Frean, Leary, Lewis, Ma, and
  McWilliams]{pmlr-v70-balduzzi17b}
Balduzzi, D., Frean, M., Leary, L., Lewis, J.P., Ma, K.W., and McWilliams, B.
\newblock The shattered gradients problem: If resnets are the answer, then what
  is the question?
\newblock In \emph{Proceedings of the 34th International Conference on Machine
  Learning}, volume~70, pp.\  342--350, 2017.

\bibitem[Barker et~al.(2017)Barker, Marxer, Vincent, and
  Watanabe]{Barker2017csl}
Barker, J., Marxer, R., Vincent, E., and Watanabe, S.
\newblock The third ‘chime’ speech separation and recognition challenge:
  Analysis and outcomes.
\newblock \emph{Computer Speech and Language}, 46:\penalty0 605--626, 2017.

\bibitem[Bengio et~al.(1994)Bengio, Simard, and Frasconi]{bengio1994learning}
Bengio, Y., Simard, P., and Frasconi, P.
\newblock Learning long-term dependencies with gradient descent is difficult.
\newblock \emph{IEEE transactions on neural networks}, 5\penalty0 (2):\penalty0
  157--166, 1994.

\bibitem[Berthelot et~al.(2017)Berthelot, Schumm, and Metz]{berthelot2017began}
Berthelot, D., Schumm, T., and Metz, L.
\newblock Began: Boundary equilibrium generative adversarial networks.
\newblock \emph{arXiv preprint arXiv:1703.10717}, 2017.

\bibitem[Billingsley(1995)]{Billingsley}
Billingsley, P.
\newblock \emph{Probability and Measure}.
\newblock Wiley-Interscience, 3rd edition, 1995.
\newblock ISBN 0471007102.

\bibitem[Bryc(1995)]{bryc1995rotation}
Bryc, W.
\newblock Rotation invariant distributions.
\newblock In \emph{The Normal Distribution}, pp.\  51--69. Springer, 1995.

\bibitem[Burgess(1997)]{burgess1997estimating}
Burgess, A.N.
\newblock Estimating equivalent kernels for neural networks: A data
  perturbation approach.
\newblock In \emph{Advances in Neural Information Processing Systems}, pp.\
  382--388, 1997.

\bibitem[Cho \& Saul(2009)Cho and Saul]{cho2009kernel}
Cho, Y. and Saul, L.K.
\newblock Kernel methods for deep learning.
\newblock In \emph{Advances in Neural Information Processing Systems}, pp.\
  342--350, 2009.

\bibitem[Choromanska et~al.(2015)Choromanska, Henaff, Mathieu, Arous, and
  LeCun]{choromanska2015loss}
Choromanska, A., Henaff, M., Mathieu, M., Arous, G.B., and LeCun, Y.
\newblock The loss surfaces of multilayer networks.
\newblock In \emph{Artificial Intelligence and Statistics}, pp.\  192--204,
  2015.

\bibitem[Clevert et~al.(2016)Clevert, Unterthiner, and
  Hochreiter]{clevert2015fast}
Clevert, D., Unterthiner, T., and Hochreiter, S.
\newblock Fast and accurate deep network learning by exponential linear units
  (elus).
\newblock In \emph{International Conference on Learning Representations}, 2016.

\bibitem[Daniely et~al.(2016)Daniely, Frostig, and Singer]{daniely2016toward}
Daniely, A., Frostig, R., and Singer, Y.
\newblock Toward deeper understanding of neural networks: The power of
  initialization and a dual view on expressivity.
\newblock In \emph{Advances In Neural Information Processing Systems}, pp.\
  2253--2261, 2016.

\bibitem[Glorot \& Bengio(2010)Glorot and Bengio]{glorot2010understanding}
Glorot, X. and Bengio, Y.
\newblock Understanding the difficulty of training deep feedforward neural
  networks.
\newblock In \emph{Proceedings of the Thirteenth International Conference on
  Artificial Intelligence and Statistics}, pp.\  249--256, 2010.

\bibitem[Haeffele \& Vidal(2015)Haeffele and Vidal]{haeffele2015global}
Haeffele, B.D. and Vidal, R.
\newblock Global optimality in tensor factorization, deep learning, and beyond.
\newblock \emph{arXiv preprint arXiv:1506.07540}, 2015.

\bibitem[He et~al.(2015)He, Zhang, Ren, and Sun]{he2015delving}
He, K., Zhang, X., Ren, S., and Sun, J.
\newblock Delving deep into rectifiers: Surpassing human-level performance on
  imagenet classification.
\newblock In \emph{Proceedings of the IEEE international conference on computer
  vision}, pp.\  1026--1034, 2015.

\bibitem[Hochreiter(1991)]{hochreiter1991untersuchungen}
Hochreiter, S.
\newblock Untersuchungen zu dynamischen neuronalen netzen.
\newblock \emph{Diploma, Technische Universit{\"a}t M{\"u}nchen}, 91, 1991.

\bibitem[Hochreiter et~al.(2001)Hochreiter, Bengio, and
  Frasconi]{hochreiter2001gradient}
Hochreiter, S., Bengio, Y., and Frasconi, P.
\newblock Gradient flow in recurrent nets: the difficulty of learning long-term
  dependencies.
\newblock In Kolen, J. and Kremer, S. (eds.), \emph{Field Guide to Dynamical
  Recurrent Networks}. IEEE Press, 2001.

\bibitem[Huang et~al.(2004)Huang, Zhu, and Siew]{huang2004extreme}
Huang, G., Zhu, Q., and Siew, C.
\newblock Extreme learning machine: a new learning scheme of feedforward neural
  networks.
\newblock In \emph{Neural Networks, 2004. Proceedings. 2004 IEEE International
  Joint Conference on}, volume~2, pp.\  985--990. IEEE, 2004.

\bibitem[Ioffe \& Szegedy(2015)Ioffe and Szegedy]{ioffe2015batch}
Ioffe, S. and Szegedy, C.
\newblock Batch normalization: Accelerating deep network training by reducing
  internal covariate shift.
\newblock In \emph{International Conference on Machine Learning}, pp.\
  448--456, 2015.

\bibitem[Johnson(2013)]{johnson2013multivariate}
Johnson, M.E.
\newblock \emph{Multivariate statistical simulation: A guide to selecting and
  generating continuous multivariate distributions}.
\newblock John Wiley \& Sons, 2013.

\bibitem[Jones(1982)]{jones_1982}
Jones, D.S.
\newblock \emph{The Theory of Generalised Functions}, chapter~7, pp.\  263.
\newblock Cambridge University Press, 2nd edition, 1982.

\bibitem[Krizhevsky \& Hinton(2009)Krizhevsky and
  Hinton]{krizhevsky2009learning}
Krizhevsky, Alex and Hinton, Geoffrey.
\newblock Learning multiple layers of features from tiny images.
\newblock 2009.

\bibitem[Lee et~al.(2017)Lee, Bahri, Novak, Schoenholz, Pennington, and
  Sohl-Dickstein]{lee2017gp}
Lee, J., Bahri, Y., Novak, R., Schoenholz, S.S., Pennington, J., and
  Sohl-Dickstein, J.
\newblock Deep neural networks as gaussian processes.
\newblock \emph{arXiv preprint arXiv:1611.01232}, 2017.

\bibitem[Liu et~al.(2015)Liu, Luo, Wang, and Tang]{liu2015faceattributes}
Liu, Z., Luo, P., Wang, X., and Tang, X.
\newblock Deep learning face attributes in the wild.
\newblock In \emph{Proceedings of International Conference on Computer Vision
  (ICCV)}, December 2015.

\bibitem[Livni et~al.(2017)Livni, Carmon, and Globerson]{livni2017learning}
Livni, R., Carmon, D., and Globerson, A.
\newblock Learning infinite layer networks without the kernel trick.
\newblock In \emph{International Conference on Machine Learning}, pp.\
  2198--2207, 2017.

\bibitem[MacKay(1992)]{mackay1992practical}
MacKay, D.J.C.
\newblock A practical {B}ayesian framework for backpropagation networks.
\newblock \emph{Neural Computation}, 4\penalty0 (3):\penalty0 448--472, 1992.

\bibitem[Martin \& Mahoney(2017)Martin and Mahoney]{martin2017rethinking}
Martin, C.H. and Mahoney, M.W.
\newblock Rethinking generalization requires revisiting old ideas: statistical
  mechanics approaches and complex learning behavior.
\newblock \emph{arXiv preprint arXiv:1710.09553}, 2017.

\bibitem[Mishkin \& Matas(2016)Mishkin and Matas]{mishkin2015all}
Mishkin, D. and Matas, J.
\newblock All you need is a good init.
\newblock In \emph{International Conference on Learning Representations}, 2016.

\bibitem[Neal(1994)]{neal1995bayesian}
Neal, R.M.
\newblock \emph{Bayesian Learning for Neural Networks}.
\newblock PhD thesis, University of Toronto, 1994.

\bibitem[Pandey \& Dukkipati(2014{\natexlab{a}})Pandey and
  Dukkipati]{pandey2014go}
Pandey, G. and Dukkipati, A.
\newblock To go deep or wide in learning?
\newblock In \emph{Artificial Intelligence and Statistics}, pp.\  724--732,
  2014{\natexlab{a}}.

\bibitem[Pandey \& Dukkipati(2014{\natexlab{b}})Pandey and
  Dukkipati]{pandey2014learning}
Pandey, G. and Dukkipati, A.
\newblock Learning by stretching deep networks.
\newblock In \emph{Proceedings of the 31st International Conference on Machine
  Learning (ICML-14)}, pp.\  1719--1727, 2014{\natexlab{b}}.

\bibitem[Pao et~al.(1994)Pao, Park, and Sobajic]{pao1994learning}
Pao, Y., Park, G., and Sobajic, D.J.
\newblock Learning and generalization characteristics of the random vector
  functional-link net.
\newblock \emph{Neurocomputing}, 6\penalty0 (2):\penalty0 163--180, 1994.

\bibitem[Poole et~al.(2016)Poole, Lahiri, Raghu, Sohl-Dickstein, and
  Ganguli]{poole2016exponential}
Poole, B., Lahiri, S., Raghu, M., Sohl-Dickstein, J., and Ganguli, S.
\newblock Exponential expressivity in deep neural networks through transient
  chaos.
\newblock In \emph{Advances In Neural Information Processing Systems}, pp.\
  3360--3368, 2016.

\bibitem[Raghu et~al.(2017)Raghu, Poole, Kleinberg, Ganguli, and
  Sohl-Dickstein]{pmlr-v70-raghu17a}
Raghu, M., Poole, B., Kleinberg, J., Ganguli, S., and Sohl-Dickstein, J.
\newblock On the expressive power of deep neural networks.
\newblock In Precup, D. and Teh, Y.W. (eds.), \emph{Proceedings of the 34th
  International Conference on Machine Learning}, volume~70 of \emph{Proceedings
  of Machine Learning Research}, pp.\  2847--2854, 2017.

\bibitem[Rahimi \& Recht(2008)Rahimi and Recht]{rahimi2008random}
Rahimi, A. and Recht, B.
\newblock Random features for large-scale kernel machines.
\newblock In \emph{Advances in neural information processing systems}, pp.\
  1177--1184, 2008.

\bibitem[Roux \& Bengio(2007)Roux and Bengio]{le2007continuous}
Roux, N.~Le and Bengio, Y.
\newblock Continuous neural networks.
\newblock In \emph{Artificial Intelligence and Statistics}, pp.\  404--411,
  2007.

\bibitem[Rudi \& Rosasco(2017)Rudi and Rosasco]{rudi2017generalization}
Rudi, Alessandro and Rosasco, Lorenzo.
\newblock Generalization properties of learning with random features.
\newblock In \emph{Advances in Neural Information Processing Systems}, pp.\
  3218--3228, 2017.

\bibitem[Schmidt et~al.(1992)Schmidt, Kraaijveld, and
  Duin]{schmidt1992feedforward}
Schmidt, W.F., Kraaijveld, M.A., and Duin, R.P.W.
\newblock Feedforward neural networks with random weights.
\newblock In \emph{Pattern Recognition, 1992. Vol. II. Conference B: Pattern
  Recognition Methodology and Systems, Proceedings., 11th IAPR International
  Conference on}, pp.\  1--4. IEEE, 1992.

\bibitem[Schoenholz et~al.(2017)Schoenholz, Gilmer, Ganguli, and
  Sohl-Dickstein]{schoenholz2016deep}
Schoenholz, S.S., Gilmer, J., Ganguli, S., and Sohl-Dickstein, J.
\newblock Deep information propagation.
\newblock In \emph{International Conference on Learning Representations}, 2017.

\bibitem[Shwartz-Ziv \& Tishby(2017)Shwartz-Ziv and Tishby]{shwartz2017opening}
Shwartz-Ziv, R. and Tishby, N.
\newblock {Opening the Black Box of Deep Neural Networks via Information}.
\newblock \emph{arXiv preprint arXiv:1703.00810}, 2017.

\bibitem[Silver et~al.(2017)Silver, Schrittwieser, Simonyan, Antonoglou, Huang,
  Guez, Hubert, Baker, Lai, Bolton, Chen, Lillicrap, Hui, Sifre, van~den
  Driessche, Graepel, and Hassabis]{silver2017mastering}
Silver, D., Schrittwieser, J., Simonyan, K., Antonoglou, T., Huang, A., Guez,
  A., Hubert, T., Baker, L., Lai, M., Bolton, A., Chen, Y., Lillicrap, T., Hui,
  F., Sifre, L., van~den Driessche, G., Graepel, T., and Hassabis, D.
\newblock Mastering the game of go without human knowledge.
\newblock \emph{Nature}, 550\penalty0 (7676):\penalty0 354--359, 2017.

\bibitem[Sinha \& Duchi(2016)Sinha and Duchi]{sinha2016learning}
Sinha, A. and Duchi, J.C.
\newblock Learning kernels with random features.
\newblock In \emph{Advances in Neural Information Processing Systems}, pp.\
  1298--1306, 2016.

\bibitem[Tang et~al.(2016)Tang, Deng, and Huang]{tang2016extreme}
Tang, J., Deng, C., and Huang, G.
\newblock Extreme learning machine for multilayer perceptron.
\newblock \emph{IEEE Transactions on Neural Networks and Learning Systems},
  27\penalty0 (4):\penalty0 809--821, 2016.

\bibitem[van~den Oord et~al.(2016)van~den Oord, Dieleman, Zen, Simonyan,
  Vinyals, Graves, Kalchbrenner, Senior, and Kavukcuoglu]{vanwavenet}
van~den Oord, A., Dieleman, S., Zen, H., Simonyan, K., Vinyals, O., Graves, A.,
  Kalchbrenner, N., Senior, A., and Kavukcuoglu, K.
\newblock Wavenet: A generative model for raw audio.
\newblock \emph{arXiv preprint arXiv:1609.03499}, 2016.

\bibitem[Vincent et~al.(2017)Vincent, Watanabe, Nugraha, Barker, and
  Marxer]{Vincent2017csl}
Vincent, E., Watanabe, S., Nugraha, A., Barker, J., and Marxer, R.
\newblock An analysis of environment, microphone and data simulation mismatches
  in robust speech recognition.
\newblock \emph{Computer Speech and Language}, 46:\penalty0 535--557, 2017.

\bibitem[Williams(1997)]{williams1997computing}
Williams, C.K.I.
\newblock Computing with infinite networks.
\newblock In \emph{Advances in Neural Information Processing Systems}, pp.\
  295--301, 1997.

\bibitem[Zhang et~al.(2016)Zhang, Bengio, Hardt, Recht, and
  Vinyals]{zhang2016understanding}
Zhang, C., Bengio, S., Hardt, M., Recht, B., and Vinyals, O.
\newblock Understanding deep learning requires rethinking generalization.
\newblock \emph{arXiv preprint arXiv:1611.03530}, 2016.

\end{thebibliography}
\bibliographystyle{icml2018}

\newpage
\onecolumn
\appendix
\section*{Supplementary Material}
\section{Empirical Evaluation of Conditions in Hypothesis~\ref{hyp}}
We consider a sequence of datapoints of increasing $m$ by starting with a compressed low dimensional datapoint and decreasing the amount of compression, evaluating the asymptotic bound $m \max_{i=1}^m \Big| \frac{x^{(m)}_i}{\Vert \mathbf{x}^{(m)} \Vert} \Big|^4$ for each $m$. Figure~\ref{fig:asymptote2} shows plots of the asymptotic bound for two datasets.

\begin{figure}[!htbp]
\centering
\includegraphics[scale=0.16]{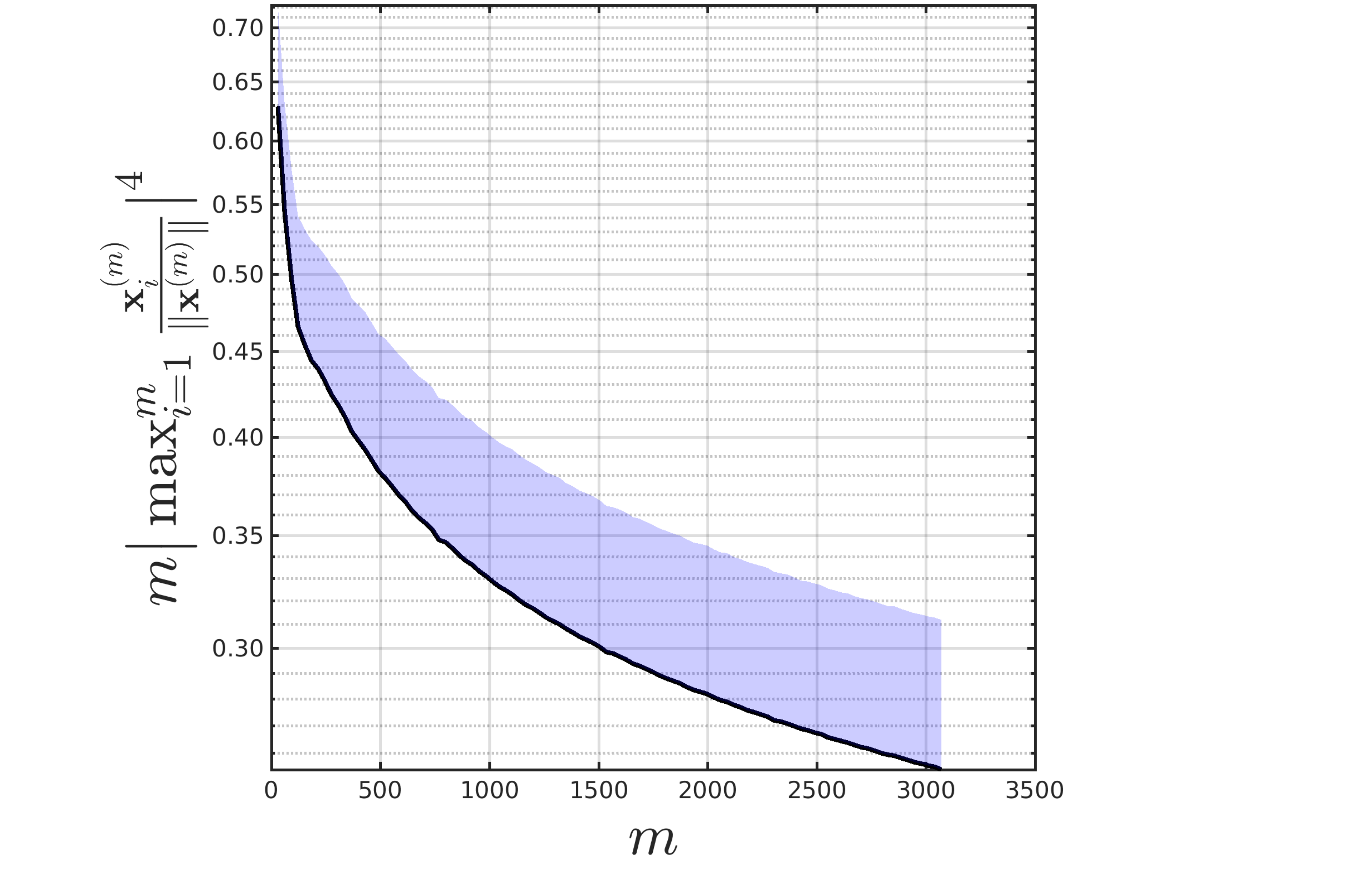}
\includegraphics[scale=0.16]{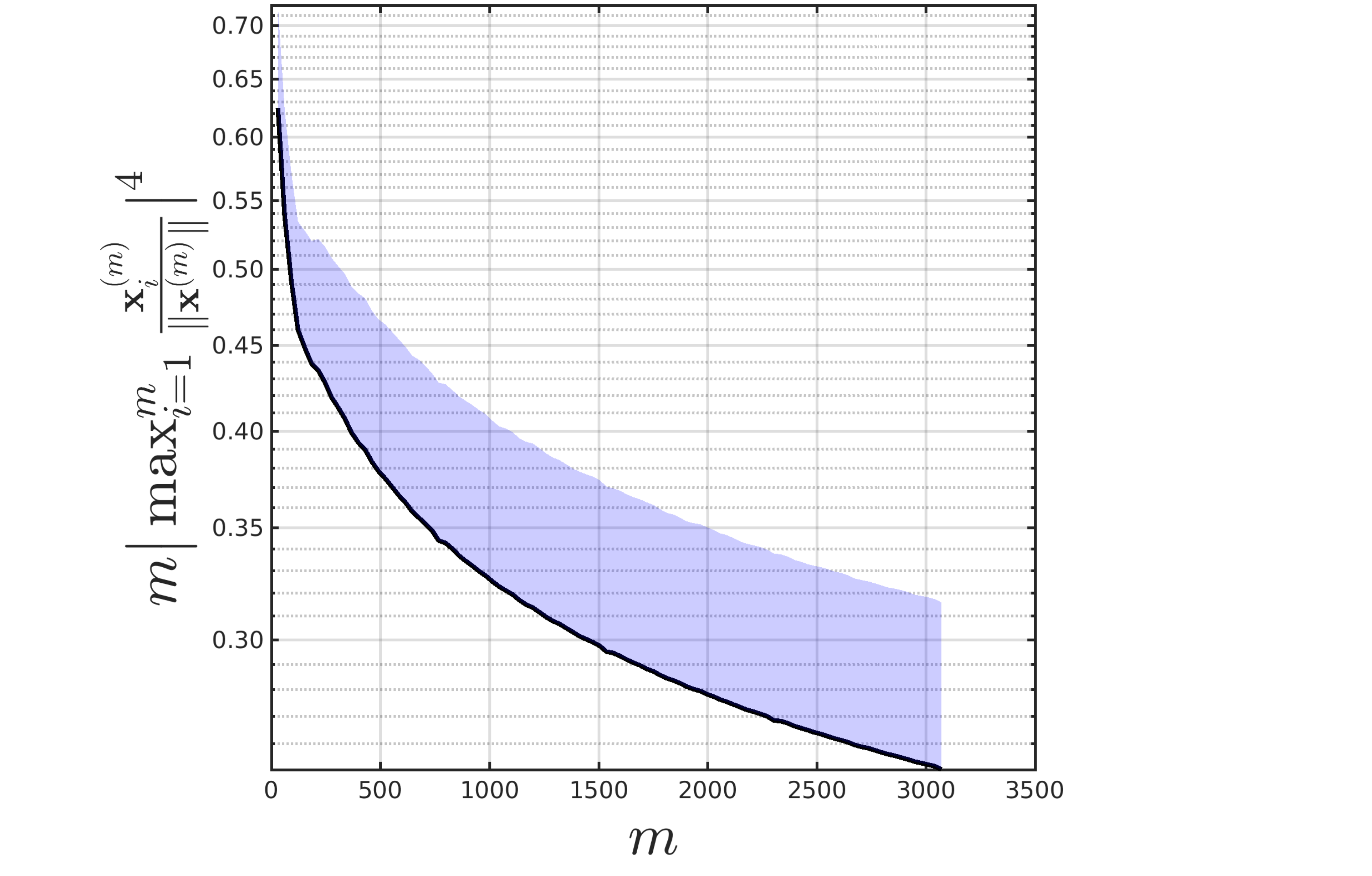}
\caption{Asymptotic error in the application of the CLT to neural network kernels. The solid line is an average over $1000$ randomly sampled datapoints and the shaded region represents $1$ standard deviation in the worst-case direction. Data is preprocessed so that each dimension is in the range $[0, 255]$. (Left) CIFAR10 and (Right) CIFAR100~\cite{krizhevsky2009learning}. The images are compressed using Bicubic Interpolation.} 
\label{fig:asymptote2}
\end{figure}

The plots suggest that Hypothesis~\ref{hyp} makes reasonable assumptions on high dimensional datasets.

\section{Proof of Proposition~\ref{cor:lrelukernel}}
\label{app:lrelu}
\begin{proof}
The LReLU activation function is $\sigma(z) = \big( a+(1-a)\Theta(z) \big)z$. Expanding, we have
\begin{align*}
k(\mathbf{x}, \mathbf{y}) &= \int_{\mathbb{R}^m} \sigma( \mathbf{w} \cdot \mathbf{x} ) \sigma( \mathbf{w} \cdot \mathbf{y} ) f(\mathbf{w}) \,\mathbf{dw}, \\
&= \int_{\mathbb{R}^m} \Big( a^2 + a(1-a)\Theta(\mathbf{w} \cdot \mathbf{y}) +  a(1-a)\Theta(\mathbf{w} \cdot \mathbf{x}) + (1-a)^2 \Theta(\mathbf{w} \cdot \mathbf{x}) \Theta(\mathbf{w} \cdot \mathbf{x}) \Big) (\mathbf{w} \cdot \mathbf{x})(\mathbf{w} \cdot \mathbf{y})f(\mathbf{w}) \,\mathbf{dw}.
\end{align*}
Using linearity of the integral, we have the superposition of the four integrals $k_1=a^2 \mathbb{E}\big[  (\mathbf{W} \cdot \mathbf{x})(\mathbf{W} \cdot \mathbf{y}) \big]$, $k_2=a(1-a)\mathbb{E}\big[  \Theta(\mathbf{W} \cdot \mathbf{x})(\mathbf{W} \cdot \mathbf{x})(\mathbf{W} \cdot \mathbf{y}) \big]$, $k_3=a(1-a)\mathbb{E}\big[  \Theta(\mathbf{W} \cdot \mathbf{y})(\mathbf{W} \cdot \mathbf{x})(\mathbf{W} \cdot \mathbf{y}) \big]$ and $k_4=(1-a)^2\mathbb{E}\big[  \Theta(\mathbf{W} \cdot \mathbf{x})\Theta(\mathbf{W} \cdot \mathbf{y})(\mathbf{W} \cdot \mathbf{x})(\mathbf{W} \cdot \mathbf{y}) \big]$.

Now $k_1(\mathbf{x}, \mathbf{y})=a^2 \mathbb{E}[W_i^2] \Vert \mathbf{x} \Vert \Vert \mathbf{y} \Vert \cos\theta_0$. To see this, rotate the coordinate system as before. Then, either solve the integral directly using the fact that the weights are uncorrelated or differentiate twice and solve the homogeneous IVP with initial conditions $k(0)=a^2\mathbb{E}[W_i^2]\Vert \mathbf{x} \Vert \Vert \mathbf{y} \Vert$ and $k'(0)=0$.

After rotating the coordinate system, differentiating $k_2(\mathbf{x}, \mathbf{y})$ twice results in a homogeneous IVP with $k(0)=a(1-a)\frac{\mathbb{E}[W_i^2]}{2}\Vert \mathbf{x} \Vert \Vert \mathbf{y} \Vert$ and $k'(0)=0$, the solution of which is $k_2(\mathbf{x}, \mathbf{y})=a(1-a)\frac{\mathbb{E}[W_i^2]}{2}\Vert \mathbf{x} \Vert \Vert \mathbf{y} \Vert\cos\theta_0$. Note that by symmetry, $k_2(\mathbf{x}, \mathbf{y})= k_3(\mathbf{x}, \mathbf{y})$.

The last remaining integral, $k_4(\mathbf{x}, \mathbf{y}),$ is just a multiple of the Arc-Cosine kernel.
\end{proof}

\section{Other Asymptotic Kernels}
\begin{cor}[Asmptotic Kernels: $1-\epsilon$ Exponent-Dominated Activation Functions]
\label{cor:epsiloncase}
Consider the same scenario as in Corollary~\ref{cor:relucase}, with the exception that the activation functions are replaced by some continuous $\sigma$ such that $|\sigma(z)| \leq M |z|^{1-\epsilon} $ for all $z \in \mathbb{R}$, some $\epsilon > 0$, and some $M \in (0, \infty),$ then for all $s \geq 2$
$$\lim_{m \to \infty} k^{(m)}_f \big( \mathbf{x}^{(m)},\mathbf{y}^{(m)} \big) = k^{(s)}_g \big( \mathbf{x}^{(s)},\mathbf{y}^{(s)} \big)=\mathbb{E} \big[ \sigma(Z_1) \sigma(Z_2) \big].$$
\end{cor}
\begin{proof}
We have $\lim_{m \to \infty} k^{(m)}_f \big( \mathbf{x}^{(m)},\mathbf{y}^{(m)} \big) = \lim_{m \to \infty} \mathbb{E} \big[ \sigma(Z_1^{(m)}) \sigma(Z_2^{(m)}) \big]$ and we would like to bring the limit inside the expected value. By Theorem~\ref{thm:asymptotic} and Theorem 25.12 of Billingsley~\yrcite{Billingsley}, it suffices to show that $\sigma(Z_1)\sigma(Z_2)$ is uniformly integrable. Define $h$ to be the joint PDF of $\mathbf{Z}$. As in (25.13) of Billingsley~\yrcite{Billingsley}, we have
\begin{align*}
\lim_{\alpha \to \infty} \int_{ |\sigma(z_1) \sigma(z_2)| > \alpha } |\sigma(z_1) \sigma(z_2)| h(z_1, z_2) \,dz_1 dz_2 & \leq \lim_{\alpha \to \infty} \frac{1}{\alpha^\epsilon} \mathbb{E} \Big[ \big| \sigma(Z_1) \sigma(Z_2) \big|^{1+\epsilon} \Big],
\end{align*}
so it suffices to show that $\mathbb{E} \Big[ \big| \sigma(Z_1) \sigma(Z_2) \big|^{1+\epsilon} \Big]$ is bounded. We have
\begin{align*}
\mathbb{E}\Big[ \big|\sigma(Z_1^{(m)}) \sigma(Z_2^{(m)}) \big|^{1+\epsilon}\Big] &\leq  M^2 \mathbb{E} \Big[  \big| Z_1^{(m)} Z_2^{(m)} \big| \Big],\\ 
&\leq  M^2 \sqrt{ \mathbb{E} \Big[  \big( Z_1^{(m)} \big)^2 \Big] \mathbb{E} \Big[  \big( Z_2^{(m)} \big)^2\Big] }, \\
&= M^2 \mathbb{E}[W_i^2] \Vert \mathbf{x} \Vert \Vert \mathbf{y} \Vert < \infty,
\end{align*}
and so \begin{align*}
\lim_{m\to \infty}k^{(m)}_f(\mathbf{x}, \mathbf{y}) =\mathbb{E} \big[  \lim_{m \to \infty}  \sigma(Z_1^{(m)}) \sigma(Z_2^{(m)}) \big] = \mathbb{E} \big[ \sigma(Z_1) \sigma(Z_2) \big].
\end{align*}
\end{proof}

\begin{cor}[Asymptotic Kernels: Bounded and Continuous Activation Functions]
\label{cor:bounded}
Consider the same scenario as in Corollary~\ref{cor:relucase}, with the exception that the activation functions are replaced by some bounded, continuous $\sigma$. Then for all $s \geq 2$
$$\lim_{m \to \infty} k^{(m)}_f \big( \mathbf{x}^{(m)},\mathbf{y}^{(m)} \big) = k^{(s)}_g \big( \mathbf{x}^{(s)},\mathbf{y}^{(s)} \big)=\mathbb{E} \big[ \sigma(Z_1) \sigma(Z_2) \big].$$
\end{cor}
\begin{proof}
This is a direct application of the Portmanteau Lemma to the result in Theorem~\ref{thm:asymptotic}:  
\begin{align*}
&\qquad\quad \Big[\sigma(\mathbf{W}^{(m)} \cdot \mathbf{x}^{(m)})\sigma( \mathbf{W}^{(m)} \cdot \mathbf{y}^{(m)}) \xrightarrow{D} \sigma(Z_1) \sigma(Z_2)\Big] \\
&\implies \Big[ \mathbb{E}\big[ \sigma(\mathbf{W}^{(m)} \cdot \mathbf{x}^{(m)})\sigma( \mathbf{W}^{(m)} \cdot \mathbf{y}^{(m)}) \big] \to \mathbb{E} \big[ \sigma(Z_1)\sigma( Z_2) \big]  \Big]
\end{align*}
for all bounded, continuous $\sigma$.
\end{proof}

\begin{cor}[Asymptotic Kernels: LReLU]
\label{cor:lrelucase}
Consider the same scenario as in Corollary~\ref{cor:relucase}, with the exception that the activation functions are replaced by the Leaky ReLU $\sigma(z)=\Theta(z)z + a\Theta(-z)z, \quad a \in (0, 1).$ Then for all $s \geq 2$
$$\lim_{m \to \infty} k^{(m)}_f \big( \mathbf{x}^{(m)},\mathbf{y}^{(m)} \big) = k^{(s)}_g \big( \mathbf{x}^{(s)},\mathbf{y}^{(s)} \big)=\mathbb{E} \big[ \sigma(Z_1) \sigma(Z_2) \big].$$
\end{cor}
\begin{proof}
As before, it suffices to show uniform integrability of the random variable $\sigma\big(Z_1^{(m)} \big)\sigma\big(Z_2^{(m)} \big)=\Theta \big(Z_1^{(m)} \big) \Theta \big(Z_2^{(m)} \big)Z_1^{(m)}Z_2^{(m)} + a\Theta \big(-Z_1^{(m)} \big) \Theta \big(Z_2^{(m)} \big)Z_1^{(m)}Z_2^{(m)} + \Theta \big(Z_1^{(m)} \big) \Theta \big(-Z_2^{(m)} \big)Z_1^{(m)}Z_2^{(m)} + a^2 \Theta \big(-Z_1^{(m)} \big) \Theta \big(-Z_2^{(m)} \big)Z_1^{(m)}Z_2^{(m)}$. Each of these terms taken individually is uniformly integrable by the same argument as in Corollary~\ref{cor:relucase}. A linear combination of uniformly integrable random variables is uniformly integrable. Thus the random variable is uniformly integrable and as before the result holds.
\end{proof}

\begin{cor}[Asymptotic Kernels: ELU]
Consider the same scenario as in Corollary~\ref{cor:relucase}, with the exception that the activation functions are replaced by the ELU $\sigma(z)=\Theta(z)z + \Theta(-z)(e^z-1).$ Then for all $s \geq 2$
$$\lim_{m \to \infty} k^{(m)}_f \big( \mathbf{x}^{(m)},\mathbf{y}^{(m)} \big) = k^{(s)}_g \big( \mathbf{x}^{(s)},\mathbf{y}^{(s)} \big)=\mathbb{E} \big[ \sigma(Z_1) \sigma(Z_2) \big].$$
\end{cor}
\begin{proof}
By Theorem~\ref{thm:asymptotic}, the random variable  $\sigma\big(Z_1^{(m)} \big)\sigma\big(Z_2^{(m)} \big)=\Theta \big(Z_1^{(m)} \big) \Theta \big(Z_2^{(m)} \big)Z_1^{(m)}Z_2^{(m)} + \Theta \big(Z_1^{(m)} \big)Z_1^{(m)}\Theta \big(-Z_2^{(m)} \big) \big( e^{Z_2^{(m)}}-1 \big) + \Theta \big(Z_2^{(m)} \big)Z_2^{(m)}\Theta \big(-Z_1^{(m)} \big) \big( e^{Z_1^{(m)}}-1 \big) + \Theta \big(-Z_1^{(m)} \big) \big( e^{Z_1^{(m)}}-1 \big)\Theta \big(-Z_2^{(m)} \big) \big( e^{Z_2^{(m)}}-1 \big)$ converges in distribution to $\sigma\big(Z_1 \big)\sigma\big(Z_2).$ Call these terms $T_1^{(m)}, T_2^{(m)}, T_3^{(m)},$ and  $T_4^{(m)}$ respectively. Due to linearity of the limit and $\mathbb{E}$, 
$$\lim_{m \to \infty} k^{(m)}_f \big( \mathbf{x}^{(m)},\mathbf{y}^{(m)} \big) = \lim_{m\to\infty} \mathbb{E} [T_1^{(m)}]+\lim_{m\to\infty} \mathbb{E} [T_2^{(m)}]+\lim_{m\to\infty} \mathbb{E} [T_3^{(m)}]+\lim_{m\to\infty} \mathbb{E} [T_4^{(m)}].$$
The first term converges by Corollary~\ref{cor:relucase}. The fourth term converges by Corollary~\ref{cor:bounded}. The quantity of interest for uniform integrability in the second and third term is the limit as $\alpha \to \infty$
\begin{align*}
\Big( \int_{\substack{z_1|e^{z_2}-1|>\alpha \\ z_1 > 0 \\ z_2 < 0}} z_1|e^{z_2}-1|h(z_1, z_2) \,dz_1 dz_2 \Big)^2 &= \Big( \mathbb{E} \big[ z_1 |e^{z_2}-1| \Theta(z_1) \Theta(-z_2) \Theta\big(z_1|e^{z_2}-1| - \alpha \big) \big] \Big)^2, \\
&\leq \mathbb{E} \Big[ z_1^2 \Theta(z_1) \Theta(-z_2) \Theta\big(z_1|e^{z_2}-1| - \alpha \big) \Big] \\
&\quad  \mathbb{E} \Big[ (e^{z_2}-1)^2  \Theta(z_1) \Theta(-z_2) \Theta\big(z_1|e^{z_2}-1| - \alpha \big) \Big],
\end{align*}
By the Monotone Convergence Theorem, the first factor evaluates as $0$ in the limit using the same argument as in Corollary~\ref{cor:relucase}. The second factor is at least bounded by $1$ because the argument of $\mathbb{E}$ is always less than $1$. So we have uniform integrability, and $\lim_{m\to\infty} \mathbb{E} [T_2^{(m)}]$ and $\lim_{m\to\infty} \mathbb{E} [T_3^{(m)}]$ converge.
\end{proof}

\section{Relation to Other Work}
From Theorem 6 of the main text, we have that
$$\lim_{m \to \infty} k_f^{(m)}(\theta_0) = \lim_{m \to \infty} \mathbb{E} \big[ \sigma(Z_1^{(m)}) \sigma(Z_2^{(m)}) \big], \quad (Z^{(m)}_1, Z^{(m)}_2)^T \xrightarrow{D} N\big(\mathbf{0}, \Sigma\big),$$
with $\Sigma =  \mathbb{E}[W_i^2]
\begin{bmatrix}
    \Vert \mathbf{x} \Vert^2       & \Vert \mathbf{x} \Vert\Vert \mathbf{y} \Vert\cos\theta_0 \\
    \Vert \mathbf{x} \Vert\Vert \mathbf{y} \Vert\cos\theta_0       & \Vert \mathbf{y} \Vert^2
\end{bmatrix}$.
If the limit could be moved inside the expectation, the right hand side would resemble the definition of a dual activation, given by Daniely et al.~\yrcite{daniely2016toward}, which follows naturally from the definition of the kernel for the special case of \emph{Gaussian} weights. We have shown that asymptotically for certain activation functions, the limit can indeed be moved inside the expectation and a large class of weight distributions may be treated as Gaussian. Therefore, much of the dual activation results apply to random neural networks operating on high dimensional data from a wide range of distributions.

\newpage
\end{document}